\documentclass[letterpaper,11pt]{article}
\usepackage[utf8]{inputenc}
\usepackage[margin=1in]{geometry}

\usepackage{float,framed}

\usepackage[colorlinks,linkcolor=blue,citecolor=blue]{hyperref}
\usepackage[round]{natbib}
\usepackage{enumitem}
\usepackage{bm}
\usepackage{xspace}

\usepackage{amsmath,amsfonts,amssymb}
\usepackage{mathtools}
\usepackage{amsthm} 
\usepackage{latexsym}

\usepackage[capitalise]{cleveref}
\usepackage{xcolor}
\usepackage{dsfont}

\newcommand{\ignore}[1]{}

\theoremstyle{plain}
\newtheorem{theorem}{Theorem}
\newtheorem{lemma}[theorem]{Lemma}
\newtheorem{corollary}[theorem]{Corollary}

\newtheorem*{theorem*}{Theorem}
\newtheorem*{lemma*}{Lemma}
\newtheorem*{corollary*}{Corollary}
\newtheorem*{proposition*}{Proposition}
\newtheorem*{claim*}{Claim}
\newtheorem*{fact*}{Fact}
\newtheorem*{observation*}{Observation}

\theoremstyle{definition}

\newtheorem*{definition*}{Definition}
\newtheorem*{remark*}{Remark}
\newtheorem*{example*}{Example}

 \theoremstyle{plain}
\newtheorem*{theoremaux}{\theoremauxref}
\gdef\theoremauxref{1}

\newenvironment{repthm}[2][]{%
  \def\theoremauxref{\cref{#2}}
  \begin{theoremaux}[#1]
}{%
  \end{theoremaux}
}

\DeclareMathAlphabet{\mathbfsf}{\encodingdefault}{\sfdefault}{bx}{n}

\DeclareMathOperator*{\argmin}{arg\,min}
\DeclareMathOperator*{\argmax}{arg\,max}

\DeclareMathOperator*{\supp}{supp}
\let\Pr\relax
\DeclareMathOperator{\Pr}{\mathbb{P}}

\newcommand{\mycases}[4]{{
\left\{
\begin{array}{ll}
    {#1} & {\;\text{#2}} \\[1ex]
    {#3} & {\;\text{#4}}
\end{array}
\right. }}

\newcommand{\mythreecases}[6] {{
\left\{
\begin{array}{ll}
    {#1} & {\;\text{#2}} \\[1ex]
    {#3} & {\;\text{#4}} \\[1ex]
    {#5} & {\;\text{#6}}
\end{array}
\right. }}

\newcommand{\lr}[1]{\!\left(#1\right)\!}

\newcommand{\lrBig}[1]{\Big(#1\Big)}
\newcommand{\lrbra}[1]{\!\left[#1\right]\!}

\newcommand{\lrset}[1]{\left\{#1\right\}}

\newcommand{\set}[1]{\{#1\}}
\newcommand{\abs}[1]{|#1|}
\newcommand{\norm}[1]{\|#1\|}
\newcommand{\ceil}[1]{\lceil #1 \rceil}
\newcommand{\floor}[1]{\lfloor #1 \rfloor}

\renewcommand{\t}[1]{\smash{\tilde{#1}}}
\newcommand{\wt}[1]{\smash{\widetilde{#1}}}
\newcommand{\wh}[1]{\smash{\widehat{#1}}}
\renewcommand{\O}{O}

\newcommand{\tO}{\wt{\O}}

\newcommand{\E}{\mathbb{E}}
\newcommand{\EE}[1]{\E\lrbra{#1}}

\newcommand{\tsum}{\sum\nolimits}

\newcommand{\ind}[1]{\mathbb{I}\!\lrset{#1}}

\newcommand{\poly}{\mathrm{poly}}

\newcommand{\tr}{^{\mkern-1.5mu\mathsf{T}}}

\newcommand{\reals}{\mathbb{R}}
\newcommand{\eps}{\epsilon}

\newcommand{\del}{\delta}
\newcommand{\Del}{\Delta}

\newcommand{\half}{\frac{1}{2}}

\newcommand{\eq}{~=~}
\renewcommand{\leq}{~\le~}
\renewcommand{\geq}{~\ge~}

\let\oldnabla\nabla
\renewcommand{\nabla}{\oldnabla\mkern-2.5mu}

\let\oldtfrac\tfrac
\renewcommand{\tfrac}[2]{\smash{\oldtfrac{#1}{#2}}}

\floatstyle{plain}
\newfloat{algorithm}{h}{lop}
\floatname{algorithm}{Algorithm}

\newcommand{\wrapalgo}[1]
{%
\begin{center}\setlength{\fboxsep}{5pt}\fbox{\begin{minipage}{0.85\linewidth}
#1
\end{minipage}}\end{center}%
}

\def\mwa{\ensuremath{\textsc{MW}^{1}}\xspace}
\def\mwb{\ensuremath{\textsc{MW}^{2}}\xspace}
\def\mwc{\ensuremath{\textsc{MW}^{3}}\xspace}
\def\leaders{\ensuremath{\textsc{Leaders}}\xspace}

\newcommand{\K}{\mathcal{K}}
\newcommand{\F}{\mathcal{F}}
\newcommand{\Y}{\mathcal{Y}}

\newcommand{\R}{R}

\newcommand{\cA}{\mathcal{A}}

\renewcommand{\H}{\mathcal{H}}
\newcommand{\X}{\mathcal{X}}

\newcommand{\Oval}{\mathsf{Val}}
\newcommand{\Oopt}{\mathsf{Opt}}
\newcommand{\br}[1]{\mathsf{BR}^{#1}}

\title{The Computational Power of Optimization in Online Learning%
}

\author{Elad Hazan\\
Princeton University\\
\texttt{ehazan@cs.princeton.edu}
 \and Tomer Koren\\
Technion\\
\texttt{tomerk@technion.ac.il} }

\date{}

\begin{document}


\maketitle

\begin{abstract}
We consider the fundamental problem of prediction with expert advice where the experts are ``optimizable'': there is a black-box optimization oracle that can be used to compute, in constant time, the leading expert in retrospect at any point in time.
In this setting, we give a novel online algorithm that attains vanishing regret with respect to $N$ experts in total~$\wt{O}(\sqrt{N})$ computation time.
We also give a lower bound showing that this running time cannot be improved (up to log factors) in the oracle model, thereby exhibiting a quadratic speedup as compared to the standard, oracle-free setting where the required time for vanishing regret is $\wt{\Theta}(N)$.
These results demonstrate an exponential gap between the power of optimization in online learning and its power in statistical learning: in the latter, an optimization oracle---i.e., an efficient empirical risk minimizer---allows to learn a finite hypothesis class of size $N$ in time~$O(\log{N})$.

We also study the implications of our results to learning in repeated zero-sum games, in a setting where the players have access to oracles that compute, in constant time, their best-response to any mixed strategy of their opponent. 
We show that the runtime required for approximating the minimax value of the game in this setting is $\wt{\Theta}(\sqrt{N})$, 
yielding again a quadratic improvement upon the oracle-free setting, where $\wt{\Theta}(N)$ is known to be tight. 
\end{abstract}

\pagenumbering{arabic}

\section{Introduction}

Prediction with expert advice is a fundamental model of sequential decision making and online learning in games.
This setting is often described as the following repeated game between a player and an adversary:
on each round, the player has to pick an expert from a fixed set of~$N$ possible experts, the adversary then reveals an arbitrary assignment of losses to the experts, and the player incurs the loss of the expert he chose to follow.
The goal of the player is to minimize his $T$-round average regret, defined as the difference between his average loss over~$T$ rounds of the game and the average loss of the best expert in that period---the one having the smallest average loss in hindsight.
Multiplicative weights algorithms (\citealp{LittlestoneW94,FreundS97}; see also \citealp{AHK-MW} for an overview) achieve this goal by maintaining weights over the experts and choosing which expert to follow by sampling proportionally to the weights; the weights are updated from round to round via a multiplicative update rule according to the observed losses.

While multiplicative weights algorithms are very general and provide particularly attractive regret guarantees that scale with $\log{N}$, they need computation time that grows linearly with~$N$ to achieve meaningful average regret.
The number of experts $N$ is often exponentially large in applications (think of the number of all possible paths in a graph, or the number of different subsets of a certain ground set), motivating the search for more structured settings where efficient algorithms are possible. 
Assuming additional structure---such as linearity, convexity, or submodularity of the loss functions---one can typically minimize regret in total $\poly(\log N)$ time in many settings of interest \citep[e.g.,][]{zinkevich2003online,kalai2005efficient,DBLP:journals/jcss/AwerbuchK08,HK12}. 
However, the basic multiplicative weights algorithm remains the most general and is still widely used.

The improvement in structured settings---most notably in the linear case \citep{kalai2005efficient} and in the convex case \citep{zinkevich2003online}---often comes from a specialized reduction of the online problem to the offline version of the optimization problem.
In other words, efficient online learning is made possible by providing access to an \emph{offline optimization oracle} over the experts, that allows the player to quickly compute the best performing expert with respect to any given distribution over the adversary's losses. 
However, in all of these cases, the regret and runtime guarantees of the reduction need the additional structure.
Thus, it is natural to ask whether such a drastic improvement in runtime is possible for generic online learning.
Specifically, we ask: 
\emph{What is the runtime required for minimizing regret given a black-box optimization oracle for the experts, without assuming any additional structure? Can one do better than linear time in~$N$?}

In this paper, we give a precise answer to these questions.
We show that, surprisingly, an offline optimization oracle gives rise to a substantial, \emph{quadratic improvement} in the runtime required for convergence of the average regret.
We give a new algorithm that is able to minimize regret in total time $\tO(\sqrt{N})$,%
\footnote{Here and throughout, we use the $\tO(\cdot)$ notation to hide constants and poly-logarithmic factors.}
and provide a matching lower bound confirming that this is, in general, the best possible. 
Thus, our results establish a tight characterization of the computational power of black-box optimization in online learning. 
In particular, unlike in many of the structured settings where $\poly(\log{N})$ runtime is possible, without imposing additional structure a polynomial dependence on $N$ is inevitable.

Our results demonstrate an exponential gap between the power of optimization in  online learning, and its power in statistical learning.
It is a simple and well-known fact that for a finite hypothesis class of size $N$ (which corresponds to a set of $N$ experts in the online setting), black-box optimization gives rise to a statistical learning algorithm---often called empirical risk minimization---that needs only~$O(\log N)$ examples for learning.
Thus, given an offline optimization oracle that optimizes in constant time, statistical learning can be performed in time $O(\log{N})$; in contrast, our results show that the complexity of online learning using such an optimization oracle is $\wt{\Theta}(\sqrt{N})$. 
This dramatic gap is surprising due to a long line of work in online learning suggesting that whatever can be done in an offline setting can also be done (efficiently) online.

Finally, we study the implication of our results to repeated game playing in two-player zero-sum games. 
The analogue of an optimization oracle in this setting is a \emph{best-response oracle} for each of the players, that allows her to quickly compute the pure action being the best-response to any given mixed strategy of her opponent. 
In this setting, we consider the problem of approximately solving a zero-sum game---namely finding a mixed strategy profile with payoff close to the minimax payoff of the game. 
We show that our new online learning algorithm above, if deployed by each of the players in an $N \times N$ zero-sum game, guarantees convergence to an approximate equilibrium in total~$\tO(\sqrt{N})$ time. 
This is, again, a quadratic improvement upon the best possible $\wt{\Theta}(N)$ runtime in the oracle-free setting, as established by \cite{GK95} and \cite{freund1999adaptive}. 
Interestingly, it turns out that the quadratic improvement is tight for solving zero-sum games as well: we prove that any algorithm would require $\wt{\Omega}(\sqrt{N})$ time to approximate the value of a zero-sum game in general, even when given access to powerful best-response oracles.

\subsection{Related Work}

\paragraph{Online-to-offline reductions.}

The most general reduction from regret minimization to optimization was introduced in the influential work of \cite{kalai2005efficient} as the Follow-the-Perturbed Leader (FPL) methodology.
This technique requires the problem at hand to be embeddable in a low-dimensional space and the cost functions to be linear in that space.%
\footnote{The extension to convex cost functions is straightforward \citep[see, e.g.,][]{ocobook}.}
Subsequently, \cite{KakadeKL09} reduced regret minimization to approximate linear optimization.
For general convex functions, the Follow-the-Regularized-Leader~(FTRL) framework (\citealp{zinkevich2003online}; see also \citealp{ocobook}) provides a general reduction from online to offline optimization, that often gives dimension-independent convergence rates.
Another general reduction was suggested by \cite{kakade2006batch} for the related model of \emph{transductive} online learning, where future data is partially available to the player (in the form of unlabeled examples).

Without a fully generic reduction from online learning to optimization, specialized online variants for numerous optimization scenarios have been explored. 
This includes efficient regret-minimization algorithms for 
online variance minimization \citep{WK06}, 
routing in networks \citep{DBLP:journals/jcss/AwerbuchK08}, 
online permutations and ranking \citep{Helmbold2009}, 
online planning \citep{DBLP:journals/mor/Even-DarKM09}, 
matrix completion \citep{HazanKS12}, 
online submodular minimization \citep{HK12}, 
contextual bandits \citep{DudikHKKLRZ11,icml2014c2_agarwalb14}, 
and many more.

\paragraph{Computational tradeoffs in learning.}

Tradeoffs between sample complexity and computation in statistical learning have been studied intensively in recent years \citep[e.g.,][]{Agarwal:EECS-2012-169,Shalev-Shwartz:2008,journals/jmlr/Shalev-ShwartzST12}.
However, the adversarial setting of online learning, which is our main focus in this paper, did not receive a similar attention. 
One notable exception is the seminal paper of \citet{blum1990separating} who showed that, under certain cryptographic assumptions, there exists an hypothesis class which is computationally hard to learn in the online mistake bound model but is \emph{non-properly} learnable in polynomial time in the PAC model.%
\footnote{Non-proper learning means that the algorithm is allowed to return an hypothesis outside of the hypothesis class it competes with.}
In our terminology, Blum's result show that online learning might require~$\omega(\poly(\log N))$ time, even in a case where offline optimization can be performed in $\poly(\log N)$ time, albeit non-properly (i.e., the optimization oracle is allowed to return a prediction rule which is not necessarily one of the $N$ experts).

\paragraph{Solution of zero-sum games.}

The computation of equilibria in zero-sum games is known to be equivalent to linear programming, as was first observed by von-Neumann \citep{adler}. 
A basic and well-studied question in game theory is the study of rational strategies that converge to equilibria (see \citealp{Nisan:2007} for an overview). 
\cite{freund1999adaptive} showed that in zero-sum games, no-regret algorithms converge to equilibrium. 
\cite{hart2000simple} studied convergence of no-regret algorithms to correlated equilibria in more general games; \cite{Even-dar09} analyzed convergence to equilibria in concave games.
\cite{GK95} were the first to observe that zero-sum games can be solved in total time sublinear in the size of the game matrix. 

Game dynamics that rely on best-response computations have been a topic of extensive research for more than half a century, since the early days of game theory. 
Within this line of work, perhaps the most prominent dynamic is the ``fictitious play'' algorithm, in which both players repeatedly follow their best-response to the empirical distribution of their opponent's past plays.
This simple and natural dynamic was first proposed by \cite{brown1951iterative}, shown to converge to equilibrium in two-player zero-sum games by \cite{Robinson}, and was extensively studied ever since (see e.g., \citealp{brandt2013rate,daskalakis2014counter} and the references therein).
Another related dynamic, put forth by \cite{Hannan57} and popularized by \cite{kalai2005efficient}, is based on perturbed (i.e., noisy) best-responses.

We remark that since the early works of \cite{GK95} and \cite{freund1999adaptive}, faster algorithms for approximating equilibria in zero-sum games have been proposed \citep[e.g.,][]{nesterov2005smooth,Daskalakis11}. 
However, the improvements there are in terms of the approximation parameter $\eps$ rather than the size of the game $N$. 
It is a simple folklore fact that using only value oracle access to the game matrix, any algorithm for approximating the equilibrium must run in time $\Omega(N)$; see, e.g., \cite{CHW}.

\section{Formal Setup and Statement of Results} 
\label{sec:model}

We now formalize our computational oracle-based model for learning in games---a setting which we call \emph{``Optimizable Experts''}. 
The model is essentially the classic online learning model of prediction with expert advice augmented with an offline optimization oracle.

Prediction with expert advice can be described as a repeated game between a player and an adversary, characterized by a finite set $\X$ of $N$ experts for the player to choose from, a set~$\Y$ of actions for the adversary, and a loss function $\ell : \X \times \Y \mapsto [0,1]$.
%
First, before the game begins, the adversary picks an arbitrary sequence $y_{1},y_{2},\ldots$ of actions from $\Y$.%
\footnote{Such an adversary is called \emph{oblivious}, since it cannot react to the decisions of the player as the game progresses. We henceforth assume an oblivious adversary, and relax this assumption later in \cref{sec:games}.}
On each round $t=1,2,\ldots,$ of the game, the player has to choose (possibly at random) an expert $x_{t} \in \X$, the adversary then reveals his action $y_{t} \in \Y$ and the player incurs the loss $\ell(x_t,y_t)$.
The goal of the player is to minimize his expected \emph{average regret} over $T$ rounds of the game, defined as
$$
	\R(T)
\eq
	\EE{ \frac{1}{T} \sum_{t=1}^{T} \ell(x_{t},y_{t}) }
	\,-\, \min_{x \in \X} \, \frac{1}{T} \sum_{t=1}^{T} \ell(x,y_{t}) 
~.
$$
Here, the expectation is taken with respect to the randomness in the choices of the player.

In the optimizable experts model, we assume that the loss function $\ell$ is initially unknown to the player, and allow her to access~$\ell$ by means of two oracles: $\Oval$ and $\Oopt$.
The first oracle simply computes for each pair of actions $(x,y)$ the respective loss $\ell(x,y)$ incurred by expert $x$ when the adversary plays the action $y$.

\begin{definition*}[value oracle]
A \emph{value oracle} is a procedure $\Oval : \X \times \Y \mapsto [0,1]$ that for any action pair $x \in \X$, $y \in \Y$, returns the loss value $\ell(x,y)$ in time $O(1)$; that is,
\begin{align*}
	\forall ~ x \in \X, y \in \Y 
~,
\qquad
	\Oval(x,y) \eq \ell(x,y) ~.
\end{align*}
\end{definition*}

The second oracle is far more powerful, and allows the player to quickly compute the best performing expert with respect to any given distribution over actions from $\Y$ (i.e., any mixed strategy of the adversary).

\begin{definition*}[optimization oracle]
An \emph{optimization oracle} is a procedure $\Oopt$ that receives as input a distribution $q \in \Del(\Y)$, represented as a list of atoms $\set{ (i,q_{i}) \,:\, q_i>0 }$, and returns a best performing expert with respect to $q$ (with ties broken arbitrarily), namely
\begin{align*}
	\forall ~ q \in \Del(\Y)
~,
\qquad
	\Oopt(q)
~\in~ 
	\argmin_{x \in \X} ~ \E_{y \sim q} [ \ell(x,y) ] 
~.
\end{align*}
The oracle $\Oopt$ runs in time $\O(1)$ on any input.
\end{definition*}

Recall that our goal in this paper is to evaluate online algorithms by their \emph{runtime complexity}. 
To this end, it is natural to consider the running time it takes for the average regret of the player to drop below some specified target threshold.%
\footnote{This is indeed the appropriate criterion in algorithmic applications of online learning methods.}
Namely, for a given $\eps > 0$, we will be interested in the total \emph{computational cost} (as opposed to the number of rounds) required for the player to ensure that $\R(T) < \eps$, as a function of $N$ and $\eps$.
Notice that the number of rounds $T$ required to meet the latter goal is implicit in this view, and only indirectly affects the total runtime.


\subsection{Main Results}

We can now state the main results of the paper: a tight characterization of the runtime required for the player to converge to $\eps$ expected average regret in the optimizable experts model. 

\begin{theorem} \label{thm:online-upper}
In the optimizable experts model, there exists an algorithm that for any $\eps>0$, guarantees an expected average regret of at most $\eps$ with total runtime of $\tO(\sqrt{N}/\eps^2)$.
Specifically, \cref{alg:N14} (see \cref{sec:main-algo}) achieves $\tO(N^{1/4}/\sqrt{T})$ expected average regret over $T$ rounds, and runs in $\tO(1)$ time per round.
\end{theorem}

The dependence on the number of experts $N$ in the above result  is tight, as the following theorem shows.

\begin{theorem} \label{thm:online-lower}
Any (randomized) algorithm in the optimizable experts model cannot guarantee an expected average regret smaller than $\half$ in total time better than $\O(\sqrt{N})$.
\end{theorem}

In other words, we exhibit a quadratic improvement in the total runtime required for the average regret to converge, as compared to standard multiplicative weights schemes that require $\tO(N/\eps^2)$ time, and this improvement is the best possible.
Granted, the regret bound attained by the algorithm is inferior to those achieved by multiplicative weights methods, that depend on $N$ only logarithmically; however, when we consider the total computational cost required for convergence, the substantial improvement is evident.

Our upper bound actually applies to a model more general than the optimizable experts model, where instead of having access to an optimization oracle, the player receives information about the leading expert on each round of the game.
Namely, in this model the player observes at the end of round $t$ the \emph{leader}
\begin{align} \label{eq:leader}
	x^*_t 
\eq 
	\argmin_{x \in \X} \sum_{s=1}^t \ell(x,y_s)
\end{align}
as part of the feedback.
This is indeed a more general model, as the leader $x^*_t$ can be computed in the oracle model in amortized $\O(1)$ time, simply by calling $\Oopt(y_1,\ldots,y_t)$. (The list of actions $y_1,\ldots,y_t$ played by the adversary can be maintained in an online fashion in $\O(1)$ time per round.)
Our lower bound, however, applies even when the player has access to an optimization oracle in its full power.

Finally, we mention a simple corollary of \cref{thm:online-lower}: we obtain that the time required to attain vanishing average regret in online Lipschitz-continuous optimization in Euclidean space is exponential in the dimension, even when an oracle for the corresponding offline optimization problem is at hand.
For the precise statement of this result, see \cref{sub:lipschitz}.

\subsection{Zero-sum Games with Best-response Oracles}

In this section we present the implications of our results for repeated game playing in two-player zero-sum games.
Before we can state the results, we first recall the basic notions of zero-sum games and describe the setting formally.

A two-player zero-sum game is specified by a matrix $G \in [0,1]^{N \times N}$, in which the rows correspond to the (pure) strategies of the first player, called the row player, while the columns correspond to strategies of the second player, called the column player.
For simplicity, we restrict the attention to games in which both players have $N$ pure strategies to choose from; our results below can be readily extended to deal with games of general (finite) size.
A mixed strategy of the row player is a distribution $p \in \Delta_N$ over the rows of $G$; similarly, a mixed strategy for the column player is a distribution $q \in \Delta_N$ over the columns. 
For players playing strategies $(p,q)$, the loss (respectively payoff) suffered by the row (respectively column) player is given by $p\tr G q$. 
A pair of mixed strategies $(p,q)$ is said to be an approximate equilibrium, if for both players there is almost no incentive in deviating from the strategies $p$ and $q$.
Formally, $(p,q)$ is an $\eps$-equilibrium if and only if
\begin{align*}
	\forall ~ 1 \le i,j \le N
~,\qquad 
	p\tr G e_j - \eps
\leq
	p\tr G q 
\leq 
	e_i\tr G q + \eps
~.
\end{align*}
Here and throughout, $e_i$ stands for the $i$'th standard basis vector, namely a vector with $1$ in its $i$'th coordinate and zeros elsewhere.
The celebrated von-Neumann minimax theorem asserts that for any zero-sum game there exists an exact equilibrium (i.e., with $\eps=0$) and it has a unique value, given by 
$$
	\lambda(G)
\eq
	\min_{p \in \Del_{N}} \max_{q \in \Del_{N}} \, p\tr G q ~.
$$ 

A repeated zero-sum game is an iterative process in which the two players simultaneously announce their strategies, and suffer loss (or receive payoff) accordingly. 
Given $\eps>0$, the goal of the players in the repeated game is to converge, as quickly as possible, to an $\eps$-equilibrium; in this paper, we will be interested in the \emph{total runtime} required for the players to reach an $\eps$-equilibrium, rather than the total number of game rounds required to do so.

We assume that the players do not know the game matrix $G$ in advance, and may only access it through two types of oracles, which are very similar to the ones we defined in the online learning model.
The first and most natural oracle allows the player to query the payoff for any pair of pure strategies (i.e., a pure strategy profile) in constant time. 
Formally,

\begin{definition*}[value oracle]
A \emph{value oracle} for a zero-sum game described by a matrix $G \in [0,1]^{N \times N}$ is a procedure $\Oval$ that accepts row and column indices $i,j$ as input and returns the game value for the pure strategy profile $(i,j)$, namely:
\begin{align*}
	\forall ~ 1 \le i,j \le N~, 
	\qquad
	\Oval (i,j) \eq G_{i,j} ~.
\end{align*}
The value oracle runs in time $O(1)$ on any valid input.
\end{definition*}

The other oracle we consider is the analogue of an optimization oracle in the context of games. 
For each of the players, a \emph{best-response oracle} is a procedure that computes the player's best response (pure) strategy to any mixed strategy of his opponent, given as input.

\begin{definition*}[best-response oracle]
A \emph{best-response oracle} for the row player in a zero-sum game described by a matrix $G \in [0,1]^{N \times N}$, is a procedure $\br{1}$ that receives as input a distribution $q \in \Del_{N}$, represented as a list of atoms $\set{ (i,q_{i}) \,:\, q_i>0 }$, and computes
\begin{align*}
	\forall ~ q \in \Del_{N} ~,
	\qquad
	\br{1}(q) ~\in~ \argmin_{1 \le i \le N} ~ \textrm{e}_{i}\tr G q
\end{align*}
with ties broken arbitrarily.
Similarly, a best-response oracle $\br{2}$ for the column player accepts as input a $p \in \Del_{N}$ represented as a list $\set{ (i,p_{i}) \,:\, p_i>0 }$, and computes 
\begin{align*}
	\forall ~ p \in \Del_{N} ~,
	\qquad
	\br{2}(p) ~\in~ \argmax_{1 \le j \le N} ~ p\tr G \textrm{e}_{j}
	~.
\end{align*}
Both best-response oracles return in time $O(1)$ on any input.
\end{definition*}

Our main results regarding the runtime required to converge to an approximate equilibrium in zero-sum games with best-response oracles, are the following. 

\begin{theorem} \label{thm:games-upper}
There exists an algorithm (see \cref{alg:zsg} in \cref{sec:games}) that for any zero-sum game with $[0,1]$ payoffs and for any $\eps>0$, terminates in time $\tO(\sqrt{N}/\eps^2)$ and outputs with high probability an $\eps$-approximate equilibrium. 
\end{theorem}

\begin{theorem} \label{thm:games-lower}
Any (randomized) algorithm for approximating the equilibrium of $N \times N$ zero-sum games with best-response oracles cannot guarantee with probability greater than $\frac{2}{3}$ that the average payoff of the row player is at most $\frac{1}{4}$-away from its value at equilibrium in total time better than~$\tO(\sqrt{N})$.
\end{theorem}

As indicated earlier, these results show that best-response oracles in repeated game playing  give rise again to a quadratic improvement in the runtime required for solving zero-sum games, as compared to the best possible runtime to do so without an access to best-response oracles, which scales linearly with $N$ \citep{GK95,freund1999adaptive}.

The algorithm deployed in \cref{thm:games-upper} above is a very natural one: it simulates a repeated game where both players play a slight modification of the regret minimization algorithm of \cref{thm:online-upper}, and the best-response oracle of each player serves as the optimization oracle required for the online algorithm; see \cref{sec:games} for more details.


\subsection{Overview of the Approach and Techniques}

We now outline the main ideas leading to the quadratic improvement in runtime achieved by our online algorithm of \cref{thm:online-upper}.
Intuitively, the challenge is to reduce the number of ``effective'' experts quadratically, from $N$ to roughly $\sqrt{N}$.
Since we have an optimization oracle at our disposal, it is natural to focus on the set of ``leaders''---those experts that have been best at some point in history---and try to reduce the complexity of the online problem to scale with the number of such leaders.
This set is natural considering our computational concerns:
the algorithm can obtain information on the leaders at almost no cost (using the optimization oracle, it can compute the leader on each round in only $\O(1)$ time per round), resulting with a potentially substantial advantage in terms of runtime.

First, suppose that there is a small number of leaders throughout the game, say $L = \O(\sqrt{N})$. 
Then, intuitively, the problem we face is easy: if we knew the identity of those leaders in advance, our regret would scale with $L$ and be independent of the total number of experts $N$. 
As a result, using standard multiplicative weights techniques we would be able to attain vanishing regret in total time that depends linearly on $L$, and in case $L = O(\sqrt{N})$ we would be done.
When the leaders are not known in advance, one could appeal to various techniques that were designed to deal with experts problems in which the set of experts evolves over time \citep[e.g.,][]{freund1997using,blum2007external,kleinberg2010regret,gofer}.
However, the per-round runtime of all of these methods is linear in $L$, which is prohibitive for our purposes. 
We remark that the simple ``follow the leader'' algorithm, that simply chooses the most recent leader on each round of the game, is not guaranteed to perform well in this case: the regret of this algorithm scales with the number of times the leader \emph{switches}---rather than the number of \emph{distinct} leaders---that might grow linearly with $T$ even when there are few active leaders. 

A main component in our approach is a novel online learning algorithm, called \leaders, that keeps track of the leaders in an online game, and attains $\tO(\sqrt{L/T})$ average regret in expectation with $\tO(1)$ runtime per round.
The algorithm, that we describe in detail in \cref{sub:leaders}, queries the oracles only $\tO(1)$ times per iteration and thus can be implemented efficiently. 
More formally,

\begin{theorem} 
\label{thm:leaders-intro}
The expected $T$-round average regret of the \leaders algorithm is upper bounded by $\O(\sqrt{(L/T) \log(LT)})$, where $L$ is an upper bound over the total number of distinct leaders during throughout the game.
The algorithm can be implemented in $\tO(1)$ time per round in the optimizable experts model.
\end{theorem}

As far as we know, this technique is new to the theory of regret minimization and may be of independent interest.
In a sense, it is a partial-information algorithm: it is allowed to use only a small fraction of the feedback signal (i.e., read a small fraction of the loss values) on each round, due to the time restrictions.
Nevertheless, its regret guarantee can be shown to be optimal in terms of the number of leaders $L$, even when removing the computational constraints!
The new algorithm is based on running in parallel a hierarchy of multiplicative-updates algorithms with varying look-back windows for keeping track of recent leaders. 

But what happens if there are many leaders, say $L = \Omega(\sqrt{N})$?
In this case, we can incorporate random guessing: if we sample about $\sqrt{N}$ experts, with nice probability one of them would be among the ``top'' $\sqrt{N}$  leaders. 
By competing with this small random set of experts, we can keep the regret under control, up to the point in time where at most $\sqrt{N}$ leaders remain active (in the sense that they appear as leaders at some later time).
In essence, this observation allows us to reduce the effective number of leaders back to the order of $\sqrt{N}$ and use the approach detailed above even when $L = \Omega(\sqrt{N})$, putting the \leaders algorithm into action at the point in time where the top $\sqrt{N}$ leader is encountered (without actually knowing when exactly this event occurs).
 
In order to apply our algorithm to repeated two-player zero-sum games and obtain \cref{thm:games-upper}, we first show how it can be adapted to minimize regret even when used against an \emph{adaptive adversary}, that can react to the decisions of the algorithm (as is the case in repeated games). 
Then, via standard techniques \citep{freund1999adaptive}, we show that the quadratic speedup we achieved in the online learning setting translates to similar speedup in the solution of zero-sum games. In a nutshell, we let both players use our online regret-minimization algorithm for picking their strategies on each round of the game, where they use their best-response oracles to fill the role of the optimization oracle in the optimizable experts model.


Our lower bounds (i.e., \cref{thm:online-lower,thm:games-lower}) are based on information-theoretic arguments, which can be turned into running time lower bounds in our oracle-based computational model. 
In particular, the lower bound for zero-sum games is based on a reduction to a  problem investigated by \cite{aldous1983minimization} and revisited years later by \citealt{aaronson2006lower}, and reveals interesting connections between the solution of zero-sum games and local-search problems. 
Aldous investigated the hardness of local-search problems and gave an explicit example of an efficiently-representable (random) function which is hard to minimize over its domain, even with access to a local improvement oracle.
(A local improvement oracle improves upon a given solution by searching in its local neighborhood.) 
Our reduction constructs a zero-sum game in which a best-response query amounts to a local-improvement step, and translates Aldous' query-complexity lower bound to a runtime lower bound in our model.

Interestingly, the connection to local-search problems is also visible in our algorithmic results: our algorithm for learning with optimizable experts (\cref{alg:N14}) involves guessing a ``top $\sqrt{N}$'' solution (i.e., a leader) and making $\sqrt{N}$ local-improvement steps to this solution (i.e., tracking the finalist leaders all the way to the final leader).
This is reminiscent of a classical randomized algorithm for local-search, pointed out by \cite{aldous1983minimization}.


\section{Algorithms for Optimizable Experts}
\label{sec:online}

In this section we develop our algorithms for online learning in the optimizable experts model.
Recall that we assume a more general setting where there is no optimization oracle, but instead the player observes after each round $t$ the identity of the leader $x^*_t$ (see \cref{eq:leader}) as part of the feedback on that round.
Thus, in what follows we assume that the leader $x^*_t$ is known immediately after round $t$ with no additional computational costs, and do not require the oracle $\Oopt$ any further.

To simplify the presentation, we introduce the following notation.
We fix an horizon $T>0$ and denote by $f_1,\ldots,f_T$ the sequence of loss functions induced by the actions $y_1,\ldots,y_T$ chosen by the adversary, where $f_t(\cdot) = \ell(\cdot,y_t)$ for all $t$; notice that the resulting sequence $f_1,\ldots,f_T$ is a completely arbitrary sequence of loss functions over $\X$, as both $\ell$ and the $y_t$'s are chosen adversarially.
We also fix the set of experts to $\X = [N] = \set{1,\ldots,N}$, identifying each expert with its serial index.

\subsection{The Leaders Algorithm}
\label{sub:leaders}

We begin by describing the main technique in our algorithmic results---the \leaders algorithm---which is key to proving \cref{thm:online-upper}. 
\leaders is an online algorithm designed to perform well in online learning problems with a small number of leaders, both in terms of average regret and computational costs.
The algorithm makes use of the information on the leaders $x^*_1,x^*_2,\ldots$ received as feedback to save computation time, and can be made to run in almost constant time per round (up to logarithmic factors).

\begin{algorithm}[h]

\wrapalgo{
{\bf Parameters:} $L$, $T$
\begin{enumerate}[nosep,leftmargin=0.5cm]
\item Set $\eta_0 = \sqrt{\log(2LT)}$ and $\nu = 2\eta_0 \sqrt{L/T}$
\item For all $r=1,\ldots,\ceil{\log_2{L}}$ and $s=1,\ldots,\ceil{\log_2{T}}$, initialize an instance $\cA_{r,s}$ of $\mwc(k,\eta,\gamma)$ with $k=2^i$, $\eta=\eta_0/\sqrt{2^{i+j}}$ and $\gamma = \tfrac{1}{T}$
\item Initialize an instance $\cA$ of $\mwa(\nu,\tfrac{1}{T})$ algorithm on the $\cA_{r,s}$'s as experts
\item
For $t=1,2,\ldots$:
\begin{enumerate}[nosep]
	\item Play the prediction $x_t$ of the algorithm chosen by $\cA$
	\item Observe feedback $f_t$ and the leader $x^*_t$, and update all algorithms $\cA_{r,s}$
\end{enumerate}
\end{enumerate}
}

\vspace{-0.5cm}
\caption{The \leaders algorithm.}
\label{alg:leaders}

\end{algorithm}

The \leaders algorithm is presented in \cref{alg:leaders}.
In the following theorem we state its guarantees; the theorem gives a slightly more general statement than the one presented earlier in \cref{thm:leaders-intro}, that we require for the proof of our main result.

\begin{theorem} \label{thm:leaders}
Assume that \leaders is used for prediction with expert advice (with leaders feedback) against loss functions $f_1,f_2,\ldots : [N] \mapsto [0,1]$, and that the total number of distinct leaders during a certain time period $t_0 < t \le t_1$ whose length is bounded by $T$, is at most $L$.
Then, provided the numbers $L$ and $T$ are given as input, the algorithm obtains the following regret guarantee:
\begin{align*}
	\EE{ \sum_{t=t_0+1}^{t_1} f_t(x_t) } 
	~-~ \lr{\sum_{t=1}^{t_1} f_t(x^*_{t_1}) - \sum_{t=1}^{t_0} f_t(x^*_{t_0})}
\leq
	25\sqrt{LT \log(2LT)}
~.
\end{align*}
The algorithm can be implemented to run in $\O(\log^2(L) \log(T))$ time per round. 
\end{theorem}

\cref{alg:N14} relies on two simpler online algorithms---the \mwa and \mwc algorithms---that we describe in detail later on in this section (see \cref{sub:hedge}, where we also discuss an algorithm called \mwb). 
These two algorithms are variants of the standard multiplicative weights (MW) method for prediction with expert advice.
\mwa is a rather simple adaptation of MW which is able to guarantee bounded regret in any time interval of predefined length:

\begin{repthm}{lem:mw}
Suppose that \mwa (\cref{alg:mw} below) is used for prediction with expert advice, against an arbitrary sequence of loss functions $f_{1},f_{2},\ldots : [N] \mapsto [0,1]$ over $N$ experts. 
Then, for $\gamma = \tfrac{1}{T}$ and any $\eta > 0$, its sequence of predictions $x_1,x_2,\ldots$ satisfies
\begin{align*}
	\EE{ \sum_{t=t_0}^{t_1} f_{t}(x_t) }
	&\,-\, \min_{x \in [N]} \sum_{t=t_0}^{t_1} f_{t}(x) 
\leq
	\frac{2\log(NT)}{\eta} + \eta T
\end{align*}
in any time interval $\set{t_0,\ldots,t_1}$ of length at most $T$.
The algorithm can be implemented to run in $O(N)$ time per round.
\end{repthm}

The \mwc algorithm a ``sliding window'' version of \mwa, that given a parameter $k>0$, maintains a buffer of $k$ experts that were recently ``activated''; in our context, an expert is activated on round $t$ if it is the leader at the end of that round.
\mwc competes (in terms of regret) with the $k$ most recent activated experts as long as they remain in the buffer.
Formally, 

\begin{repthm}{lem:exp3q}
Suppose that \mwc (\cref{alg:exp3q} below) is used for prediction with expert advice, against an arbitrary sequence of loss functions $f_{1},f_{2},\ldots : [N] \mapsto [0,1]$ over $N$ experts. 
Assume that expert $x^* \in [N]$ was activated on round $t_0$, and from that point until round $t_1$ there were no more than $k$ different activated experts (including $x^*$ itself).
Then, for $\gamma = \tfrac{1}{T}$ and any $\eta > 0$, the predictions $x_1,x_2,\ldots$ of the algorithm satisfy
\begin{align*}
	\EE{ \sum_{t=t_0'+1}^{t_1'} f_t(x_t) } 
	\,-\, \sum_{t=t_0'+1}^{t_1'} f_t(x^*)
\leq
	\frac{4\log(kT)}{\eta} + \eta kT
\end{align*}
in any time interval $[t_0',t_1'] \subseteq [t_0,t_1]$ of length at most $T$.
Furthermore, the algorithm can be implemented to run in time $\tO(1)$ per round.
\end{repthm}

For the analysis of \cref{alg:leaders}, we require a few definitions.
We let $I = \set{t_0+1,\ldots,t_1}$ denote the time interval under consideration.
For all $t \in I$, we denote by $S_t = \set{x^*_{t_0+1},\ldots,x^*_t}$ the set of all leaders encountered since round $t_0+1$ up to and including round $t$; for completeness we also define $S_{t_0} = \emptyset$. The theorem's assumption then implies that $\abs{S_{t_0}} \le \ldots \le \abs{S_{t_1}} \le L$.
For a set of experts $S \subseteq [N]$, we let $\tau(S) = \max\set{t \in I \,:\, x^*_t \in S}$ be the last round in which one of the experts in $S$ occurs as a leader.
In other words, after round $\tau(S)$, the leaders in $S$ have ``died out'' and no longer appear as leaders.

\begin{figure}[ht]
\begin{center}
\includegraphics[width=0.9\linewidth]{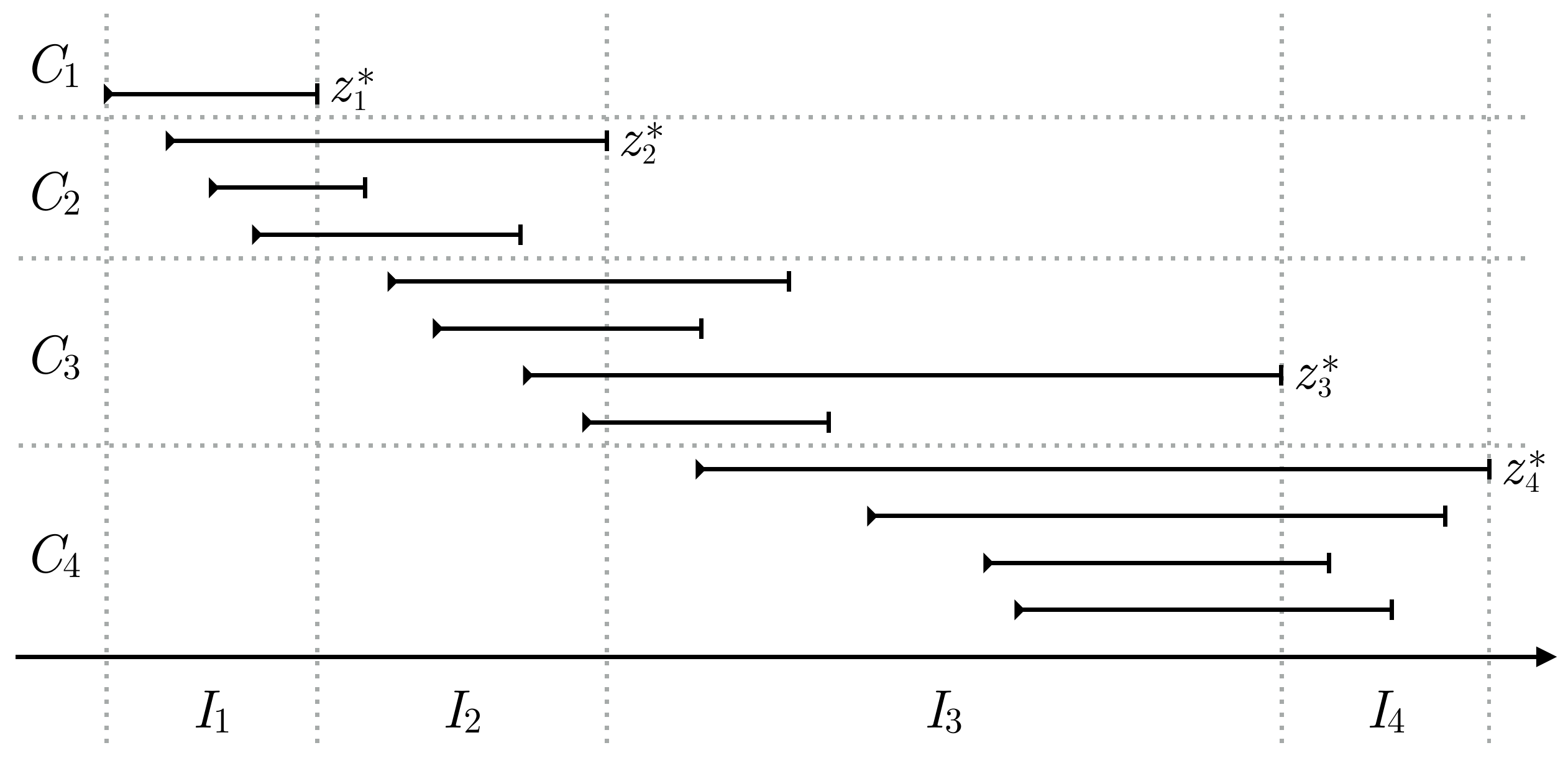}
\caption{An illustration of the key definitions in the analysis of \cref{alg:leaders}. 
Each expert is represented by a horizontal segment, which signifies the time interval between the expert's first and last appearances as leader (the experts are sorted by their first time of appearance as leaders).
The resulting partition $C_1,\ldots,C_4$ of the experts and the induced epochs $I_1,\ldots,I_4$ are indicated by the dotted lines.
}
\label{fig:leaders}
\end{center}
\end{figure}

Next, we split $I$ into epochs $I_1,I_2,\ldots$, where the $i$'th epoch $I_i = \set{\tau_i+1,\ldots,\tau_{i+1}}$ spans between rounds $\tau_i+1$ and $\tau_{i+1}$, and $\tau_i$ is defined recursively by $\tau_1 = t_0$ and $\tau_{i+1} = \tau(S_{\tau_i+1})$ for all $i=1,2,\ldots$. 
In words, $S_{\tau_i+1}$ is the set of leaders encountered by the beginning of epoch $i$, and this epoch ends once all leaders in this set have died out.
Let $m$ denote the number of resulting epochs (notice that $m \le L$, as at least one leader dies out in each of the epochs).
For each $i=1,\ldots,m$, let $T_i$ denote the length of the $i$'th epoch, namely $m = \max\set{i \,:\, \tau_{i} < t_1}$, and let $z^*_i = x^*_{\tau_i}$ be the leader at the end of epoch $i$. 
Finally, for each epoch $i=1,\ldots,m$ we let $C_i = S_{\tau_i+1} \setminus S_{\tau_{i-1}+1}$ denote the set of leaders that have died out during the epoch, and for technical convenience we also define $C_0 = C_{m+1} = \emptyset$; notice that $C_1,\ldots,C_m$ is a partition of the set $S_{t_1}$ of all leaders, so in particular $\abs{C_1} + \cdots + \abs{C_m} \le L$. 
See \cref{fig:leaders} for an illustration of the definitions.

Our first lemma states that minimizing regret in each epoch $i$ with respect to the leader $z^*_i$ at the end of the epoch, also guarantees low regret with respect to the overall leader $x^*_{t_1}$.
It is a variant of the ``Follow The Leader, Be The Leader'' lemma \citep{kalai2005efficient}.

\begin{lemma} \label{lem:btl}
Following the epoch's leader yields no regret, in the sense that
\begin{align*} 
	\sum_{i=1}^{m} \sum_{t \in I_i} f_t(z^*_i) 
\leq
	\sum_{t=1}^{t_1} f_t(x^*_{t_1}) \,-\, \sum_{t=1}^{t_0} f_t(x^*_{t_0})
~.
\end{align*}
\end{lemma}

\begin{proof}
Let $I_0 = \set{1,\ldots,t_0}$ and $z^*_0 = x^*_{t_0}$.
We will prove by induction on $m \ge 0$ that
\begin{align} \label{eq:btl-ind}
	\sum_{i=0}^{m} \sum_{t \in I_i} f_t(z^*_i) 
\leq
	\sum_{i=0}^{m} \sum_{t \in I_i} f_t(z^*_{m})
~.
\end{align}
This inequality would imply the lemma, as $z^*_m = x^*_{t_1}$.
For $m=0$, our claim is trivial as both sides of \cref{eq:btl-ind} are equal.
Now, assuming that \cref{eq:btl-ind} holds for $m-1$, we have
\begin{align*}
	\sum_{i=0}^{m-1} \sum_{t \in I_i} f_t(z^*_i) 
&\leq
	\sum_{i=0}^{m-1} \sum_{t \in I_i} f_t(z^*_{m-1}) 
	\qquad\qquad
	\mbox{(induction)}
\\
&\leq
	\sum_{i=0}^{m-1} \sum_{t \in I_i} f_t(z^*_{m})
~,
\end{align*}
since by definition $z^*_{m-1}$ performs better than any other expert, and in particular than $z^*_{m}$, throughout the first $m-1$ epochs.
Adding the term $\sum_{t \in I_m} f_t(z^*_{m})$ to both sides of the above inequality, we obtain \cref{eq:btl-ind}.
\end{proof}

Next, we identify a key property of our partition to epochs.

\begin{lemma} \label{lem:C}
For all epochs $i$, it holds that $z^*_i \in C_i$.
In addition, any leader encountered during the lifetime of $z^*_i$ as leader (i.e., between its first and last appearances in the sequence $x^*_{t_0+1},\ldots,x^*_{t_1}$ of leaders) must be a member of $C_{i-1} \cup C_i \cup C_{i+1}$.
\end{lemma}

\begin{proof}
Consider epoch $i$ and the leader $z^*_i$ at the end of this epoch.
To see that $z^*_i \in C_i$, recall that the $i$'th epoch ends right after the leaders in $C_i$ have all died out, so the leader at the end of this epoch must be a member of the latter set.
This also means that $z^*_i$ was first encountered not before epoch $i-1$ (in fact, even not on the first round of that epoch), and the last time it was a leader was on the last round of epoch $i$ (see \cref{fig:leaders}).
In particular, throughout the lifetime of $z^*_i$ as leader, only the experts in $C_{i-1} \cup C_i \cup C_{i+1}$ could have appeared as leaders.
\end{proof}

We are now ready to analyze the regret in a certain epoch $i$ with respect to its leader $z^*_i$.
To this end, we define $k_i = \abs{C_{i-1}}+\abs{C_i}+\abs{C_{i+1}}$ and consider the \mwc instance $\cA^{(i)} = \cA_{r_i,s_i}$, where $r_i = \ceil{\log_2{k_i}}$ and $s_i = \ceil{\log_2{T_i}}$ (note that $1 \le r_i \le \ceil{\log_2{L}}$ and $1 \le s_i \le \ceil{\log_2{T}}$).
The following lemma shows that the regret of the algorithm $\cA^{(i)}$ in epoch $i$ can be bounded in terms of the quantity $k_i$.
Below, we use $z_t^{(i)}$ to denote the decision of $\cA^{(i)}$ on round $t$. 

\begin{lemma} \label{lem:epoch-exp3q}
The cumulative expected regret of the algorithm $\cA^{(i)}$ throughout epoch $i$, with respect to the leader $z^*_i$ at the end of this epoch, has
\begin{align*} 
	\EE{ \sum_{t \in I_i} f_t(z^{(i)}_t) } \,-\, \sum_{t \in I_i} f_t(z^*_i)
\leq 
	10\sqrt{ k_i T_i \log(2LT) }
~.
\end{align*}
\end{lemma}

\begin{proof}
Recall that $\cA^{(i)}$ has a buffer of size $q_i = 2^{r_i}$ and step size $\eta_i = \sqrt{\log(2LT)/2^{r_i+s_i}}$.
Now, from \cref{lem:C} we know that $z^*_i \in C_i$, which means that $z^*_i$ first appeared as leader either on or before the first round of epoch $i$. Also, the same lemma states that the number of distinct leaders that were encountered throughout the lifetime of $z^*_i$ (including $z^*_i$ itself) is at most $\abs{C_{i-1} \cup C_i \cup C_{i+1}} = k_i \le q_i$, namely no more than the size of $\cA^{(i)}$'s buffer.
Hence, applying \cref{lem:exp3q} to epoch $i$, we have
\begin{align*}
	\EE{ \sum_{t \in I_i} f_t(z^{(i)}_t) } \,-\, \sum_{t \in I_i} f_t(z^*_i)
\leq 
	\frac{4\log(2LT)}{\eta_{i}} + \eta_{i} q_i T_i
~,
\end{align*}
where we have used $q_i \le 2L$ and $T_i \le T$ to bound the logarithmic term.
Now, note that $q_i \le 2k_i$ and $\sqrt{\log(2LT)/4 k_i T_i} \le \eta_i \le \sqrt{\log(2LT)/k_i T_i}$, which follow from $k_i \le 2^{r_i} \le 2k_i$ and $T_i \le 2^{s_i} \le 2T_i$. 
Plugging into the above bound, we obtain the lemma.
\end{proof}

Our final lemma analyzes the MW algorithm $\cA$, and shows that it obtains low regret against the algorithm $\cA^{(i)}$ in epoch $i$.

\begin{lemma} \label{lem:epoch-mw}
The difference between the expected cumulative loss of \cref{alg:leaders} during epoch $i$, and the expected cumulative loss of $\cA^{(i)}$ during that epoch, is bounded as
\begin{align*} 
	\EE{ \sum_{t \in I_i} f_t(x_t) \,-\, \sum_{t \in I_i} f_t(z^{(i)}_t) }
\leq 
	\frac{4\log(2LT)}{\nu} + \nu T_i
~.
\end{align*}
\end{lemma}

\begin{proof}
The algorithm $\cA$ is following \mwa updates over $m = \ceil{\log_2 L} \cdot \ceil{\log_2 T}$ algorithms as meta-experts.
Thus, \cref{lem:mw} gives
\begin{align*}
	\EE{ \sum_{t \in I_i} f_t(x_t) \,-\, \sum_{t \in I_i} f_t(z^{(i)}_t) }
\leq
	\frac{2\log(mT)}{\nu}
	+ \nu T
~.
\end{align*}
Using $m \le 2LT$ to bound the logarithmic term gives the result.
\end{proof}

We now turn to prove the theorem.

\begin{proof}[Proof of \cref{thm:leaders}]
First, regarding the running time of the algorithm, note that on each round  \cref{alg:leaders} has to update $\O(\log(L) \log(T))$ instances of \mwa, where each such update costs at most $\O(\log{L})$ time according to \cref{lem:exp3}. 
Hence, the overall runtime per round is $\O(\log^2(L) \log(T))$.

We next analyze the expected regret of the algorithm.
Summing the bounds of \cref{lem:epoch-exp3q,lem:epoch-mw} over epochs $i=1,\ldots,m$ and adding that of \cref{lem:btl}, we can bound the expected regret of \cref{alg:leaders} as follows:
\begin{align} \label{eq:sum-epochs}
	\EE{ \sum_{t=t_0+1}^{t_1} f_t(x_t) } 
	&~-~ \lr{\sum_{t=1}^{t_1} f_t(x^*_{t_1}) - \sum_{t=1}^{t_0} f_t(x^*_{t_0})}
\notag\\
&\leq
	10\sqrt{\log(2LT)} \, \sum_{i=1}^{m} \sqrt{ k_i T_i}
	+ \sum_{i=1}^{m} \lr{ \frac{4\log(2LT)}{\nu} + \nu T_i }
\notag\\
&\leq
	10\sqrt{\log(2LT)} \, \sum_{i=1}^{m} \sqrt{ k_i T_i}
	+ \frac{4L\log(2LT)}{\nu} + \nu T
~,
\end{align}
where we have used $m \le L$ and $\sum_{i=1}^m T_i = T$.
In order to bound the sum on the right-hand side, we first notice that $\sum_{i=1}^m k_i \le 3\sum_{i=1}^m \abs{C_i} \le 3L$. 
Hence, using the Cauchy-Schwarz inequality we get 
$
	\sum_i \sqrt{k_i T_i} 
\le 
	\sqrt{\sum_i k_i} \sqrt{\sum_i T_i}
\le
	\sqrt{3LT} .
$
Combining this with \cref{eq:sum-epochs} and our choice of $\nu$, and rearranging the left-hand side of the inequality, we obtain
\begin{align*}
	\EE{ \sum_{t=t_0+1}^{t_1} f_t(x_t) } 
	~-~ \lr{\sum_{t=1}^{t_1} f_t(x^*_{t_1}) - \sum_{t=1}^{t_0} f_t(x^*_{t_0})}
\leq
	25\sqrt{ LT \log(2LT) }
~,
\end{align*}
and the theorem follows.
\end{proof}

\subsection{Main Algorithm}
\label{sec:main-algo}

We now ready to present our main online algorithm: an algorithm for online learning with optimizable experts, that guarantees $\eps$ expected average regret in total $\tO(\sqrt{N}/\eps^2)$ time.
The algorithm is presented in \cref{alg:N14}, and in the following theorem we give its guarantees.

\begin{algorithm}[h]

\wrapalgo{
{\bf Parameters:} $N$, $T$
\begin{enumerate}[nosep,leftmargin=0.5cm]
\item Set $\eta = 2/(N^{1/4} \sqrt{T})$ and $\nu = 2\sqrt{\log(2T)/T}$
\item Sample a set $R$ of $\smash{\floor{2\sqrt{N}\log{T}}}$ experts uniformly at random with replacement
\item Initialize an instance $\cA_1$ of $\mwb(\eta,\tfrac{1}{T})$ on the experts in $R$
\item Initialize an instance $\cA_2$ of $\leaders(L,T)$ with $L = \floor{\sqrt{N}}$
\item Initialize an instance $\cA$ of $\mwa(\nu,\tfrac{1}{T})$ algorithm on $\cA_1$ and $\cA_2$ as experts
\item
For $t=1,2,\ldots,T$:
\begin{enumerate}[nosep]
	\item Play the prediction $x_t$ of the algorithm chosen by $\cA$
	\item Observe $f_t$ and the new leader $x^*_t$, and use them to update $\cA_1$, $\cA_2$ and $\cA$
\end{enumerate}
\end{enumerate}
}

\vspace{-0.5cm}
\caption{Algorithm for online learning with an optimization oracle.}
\label{alg:N14}

\end{algorithm}

\begin{repthm}[restated]{thm:online-upper}
The expected average regret of \cref{alg:N14} on any sequence of $T$ loss functions~$f_1,\ldots,f_T : [N] \mapsto [0,1]$ over $N$ experts is upper bounded by $40 N^{1/4}\log(NT)/\sqrt{T}$. 
The algorithm can be implemented to run in $\tO(1)$ time per round in the optimizable experts model.
\end{repthm}

\cref{alg:N14} relies on the \leaders and \mwa algorithms discussed earlier, and on yet another variant of the MW method---the \mwb algorithm---which is similar to \mwa.
The difference between the two algorithms is in their running time per round: \mwa, like standard MW, runs in $\O(N)$ time per round over $N$ experts; \mwb is an ``amortized'' version of \mwa that spreads computation over time and runs in only $\tO(1)$ time per round, but requires $N$ times more rounds to converge to the same average regret.


\begin{repthm}{lem:exp3}
Suppose that \mwb (see \cref{alg:exp3}) is used for prediction with expert advice, against an arbitrary sequence of loss functions $f_{1},f_{2},\ldots : [N] \mapsto [0,1]$ over $N$ experts.
Then, for $\gamma = \tfrac{1}{T}$ and any $\eta > 0$, its sequence of predictions $x_1,x_2,\ldots$ satisfies
\begin{align*}
	\EE{ \sum_{t=t_0}^{t_1} f_{t}(x_t)  }
	&\,-\, \min_{x \in [N]} \sum_{t=t_0}^{t_1} f_{t}(x) 
\leq
	\frac{4\log(NT)}{\eta} + \eta NT
\end{align*}
in any time interval $\set{t_0,\ldots,t_1}$ of length at most $T$.  
The algorithm can be implemented to run in $O(\log{N})$ time per round.
\end{repthm}

%

Given the \leaders algorithm, the overall idea behind \cref{alg:N14} is quite simple: first guess~$\sqrt{N}$ experts uniformly at random, so that with nice probability one of the ``top'' $\sqrt{N}$ experts is picked, where experts are ranked according to the last round of the game in which they are leaders. (In particular, the best expert in hindsight is ranked first.)
The first online algorithm $\cA_1$---an instance of \mwb---is designed to compete with this leader, up to that point in time where it appears as leader for the last time.
At this point, the second algorithm $\cA_2$---an instance of \leaders---comes into action and controls the regret until the end of the game. 
It is able to do so because in that time period there are only few different leaders (i.e., at most $\sqrt{N}$), and as we pointed out earlier, \leaders is designed to exploit this fact.
The role of the algorithm $\cA$, being executed on top of $\cA_1$ and $\cA_2$ as experts, is to combine between the two regret guarantees, each in its relevant time interval.

Using \cref{thm:leaders,lem:mw,lem:exp3}, we can formalize the intuitive idea sketched above and prove the main result of this section.

\begin{proof}[Proof of \cref{thm:online-upper}]
The fact that the algorithm can be implemented to run in $\tO(1)$ time per round follows immediately from the running time of the algorithms \mwa, \mwb, and \leaders, each of which runs in $\tO(1)$ time per round with the parameters used in \cref{alg:N14}.

We move on to analyze the expected regret.
Rank each expert $x \in [N]$ according to $\mathrm{rank}(x) = 0$ if $x$ is never a leader throughout the game, and $\mathrm{rank}(x) = \max\set{t \,:\, x^*_t  = x}$ otherwise. 
Let $x_{(1)},\ldots,x_{(N)}$ be the list of experts sorted according to their rank in decreasing order (with ties broken arbitrarily).
In words, $x_{(1)}$ is the best expert in hindsight, $x_{(2)}$ is the expert leading right before $x_{(1)}$ becomes the sole leader, $x_{(3)}$ is the leading expert right before $x_{(1)}$ and $x_{(2)}$ become the only leaders, and so on.
Using this definition, we define $X^* = \set{x_{(1)},\ldots,x_{(n)}}$ be the set of the top $n = \smash{\floor{\sqrt{N}}}$ experts having the highest rank.

First, consider the random set $R$.
We claim that with high probability, this set contains at least one of the top $n$ leaders. 
Indeed, we have
\begin{align*}
	\Pr(R \cap X^* = \emptyset)
\eq
	\lrBig{ 1-\frac{n}{N} }^{\floor{2\sqrt{N}\log{T}}}
\leq
	\lrBig{ 1-\frac{1}{2\sqrt{N}} }^{\sqrt{N}\log{T}}
\leq
	e^{-\tfrac{1}{2}\log{T}}
\eq
	\frac{1}{\sqrt{T}}
~,
\end{align*}
so that with probability at least $1-1/\sqrt{T}$ it holds that $R \cap X^* \ne \emptyset$.
As a result, it is enough to upper bound the expected regret of the algorithm for any fixed realization of $R$ such that $R \cap X^* \ne \emptyset$: in the event that the intersection is empty, that occurs with probability $1/\sqrt{T}$, the regret can be at most $T$ and thus ignoring these realizations can only affect the expected regret by an additive $\sqrt{T}$ term.
Hence, in what follows we fix an arbitrary realization of the set $R$ such that $R \cap X^* \ne \emptyset$ and bound the expected regret of the algorithm.


Given $R$ with $R \cap X^* \ne \emptyset$, we can pick $x \in R \cap X^*$ and let $T_0 = \mathrm{rank}(x)$ be the last round in which $x$ is the leader.
Since $x \in R$ and $\abs{R} \le 2\sqrt{N}\log{T}$, the \mwb instance $\cA_1$ over the experts in $R$, with parameter $\eta = 2/(N^{1/4} \sqrt{T})$, guarantees (recall \cref{lem:exp3}) that
\begin{align} \label{eq:N14-exp3}
	\EE{ \sum_{t=1}^{T_0} f_t(x^{(1)}_t) } - \sum_{t=1}^{T_0} f_t(x^*_{T_0})
\leq
	\frac{4\log(\abs{R}T)}{\eta} + \eta \abs{R}T
\leq
	8N^{1/4}\sqrt{T} \log(2NT)
~,
\end{align}
where we use $x^{(1)}_t$ to denote the decision of $\cA_1$ on round $t$.

On the other hand, observe that there are at most $n$ different leaders throughout the time interval $\set{T_0+1,\ldots,T}$, which follows from the fact that $x \in X^*$.
Thus, in light of \cref{thm:leaders}, we have
\begin{align} \label{eq:N14-lead}
	\EE{ \sum_{t=T_0+1}^{T} f_t(x^{(2)}_t) } 
	\,-\,
	\lr{ \sum_{t=1}^{T} f_t(x^*_{T}) 
		- \sum_{t=1}^{T_0} f_t(x^*_{T_0}) }
\leq
	25 N^{1/4}\sqrt{T \log(2NT)}
~,
\end{align}
where here $x^{(2)}_t$ denotes the decision of $\cA_2$ on round $t$. 

Now, since \cref{alg:N14} is playing \mwa on $\cA_1$ and $\cA_2$ as experts with parameter $\nu = 2\sqrt{\log(2T)/T}$, \cref{lem:mw} shows that
\begin{align} \label{eq:N14-mw1}
	\EE{ \sum_{t=1}^{T_0} f_t(x_t) 
		\,-\, \sum_{t=1}^{T_0} f_t(x^{(1)}_t) }
\leq
	\frac{2\log(2T)}{\nu}  + \nu T
\eq
	3\sqrt{T\log(2T)}
~,
\end{align}
and similarly,
\begin{align} \label{eq:N14-mw2}
	\EE{ \sum_{t=T_0+1}^{T} f_t(x_t) 
		\,-\, \sum_{t=T_0+1}^{T} f_t(x^{(2)}_t) }
\leq
	3\sqrt{T\log(2T)}
~.
\end{align}
Summing up \cref{eq:N14-exp3,eq:N14-lead,eq:N14-mw1,eq:N14-mw2} we obtain the regret bound
\begin{align} \label{eq:regretR}
	\EE{ \sum_{t=1}^{T} f_t(x_t) } \,-\, \sum_{t=1}^{T} f_t(x^*_T)
\leq
	39 N^{1/4}\sqrt{T} \log(NT)
\end{align}
for any fixed realization of $R$ with $R \cap X^* \ne \emptyset$.
As we explained before, the overall expected regret is larger by at most $\sqrt{T}$ than the right-hand side of \cref{eq:regretR}, and dividing through by $T$ gives the theorem. 
\end{proof}

\subsection{Multiplicative Weights Algorithms}
\label{sub:hedge}

We end the section by presenting the several variants of the Multiplicative Weights (MW) method used in our algorithms above.
For an extensive survey of the basic MW method and its applications, refer to \citealp{AHK-MW}.

\subsubsection{\mwa: Mixed MW}

The first variant, the \mwa algorithm, is designed so that its regret on any time interval of bounded length is controlled. 
The standard MW algorithm does not have such a property, because the weight it assigns to an expert might become very small if this expert performs badly, so that even if the expert starts making good decisions, it cannot regain a non-negligible weight.

Our modification of the algorithm (see \cref{alg:mw}) involves mixing in a fixed weight to the update of the algorithm, for all experts on each round, so as to keep the weights away from zero at all times.
We note that this is not equivalent to the more standard modification of mixing-in the uniform distribution to the sampling distributions of the algorithms: in our variant, it is essential that the mixed weights are fed back into the update of the algorithm so as to control its weights.

\begin{algorithm}[h]

\wrapalgo{
{\bf Parameters:} $\eta$, $\gamma$
\begin{enumerate}[nosep,leftmargin=0.5cm]
\item Initialize $w_1(x) = 1$ for all $x \in [N]$
\item
For $t=1,2,\ldots$:
\begin{enumerate}[nosep]
	\item For all $x \in [N]$ compute $q_t(x) = w_t(x)/W_t$ with $W_t = \sum_y w_t(y)$
	\item Pick $x_t \sim q_t$, play $x_t$ and receive feedback $f_t$
	\item For all $x \in [N]$, update $w_{t+1}(x) = w_t(x) e^{-\eta f_t(x)} + \tfrac{\gamma}{N} W_t$
\end{enumerate}
\end{enumerate}
}

\vspace{-0.5cm}
\caption{The \mwa algorithm.}
\label{alg:mw}

\end{algorithm}

In the following lemma we prove a regret bound for the \mwa algorithm. We prove a slightly more general result than the one we stated earlier in the section, which will become useful for the analysis of the subsequent algorithms.

\begin{lemma} \label{lem:mw}
For any sequence of loss functions $f_{1},f_{2},\ldots : [N] \mapsto \reals^+$ and for $\gamma = \tfrac{1}{T}$ and any $\eta > 0$, \cref{alg:mw} guarantees
\begin{align*}
	\EE{ \sum_{t=t_0}^{t_1} f_{t}(x_t) }
	&\,-\, \min_{x \in [N]} \sum_{t=t_0}^{t_1} f_{t}(x) 
\leq
	\frac{2\log(NT)}{\eta}
	+ \eta \EE{ \sum_{t=t_0}^{t_1} f_{t}(x_t)^2  }
\end{align*}
in any time interval $\set{t_0,\ldots,t_1}$ of length at most $T$.
The algorithm can be implemented to run in $O(N)$ time per round.
\end{lemma}
\begin{proof}
The claim regarding the runtime of the algorithm is trivial, as all the computations on a certain round can be completed in a single pass over the $N$ actions. 
Thus, we move on to analyze the regret; the proof follows the standard analysis of exponential weighting schemes.
We first write:
\begin{align*}
	\frac{W_{t+1}}{W_t}
&\eq 
	\sum_{x \in [N]} \frac{w_{t+1}(x)}{W_t}
\\
&\eq 
	\sum_{x \in [N]} \lr{ \frac{w_{t}(x)}{W_t} \, e^{-\eta f_{t}(x)} + \frac{\gamma}{N} } 
\\
&\eq 
	\gamma + \sum_{x \in [N]} q_{t}(x) \, e^{-\eta f_{t}(x)}
\\
&\leq 
	\gamma + \sum_{x \in [N]} q_{t}(x)\,\lr{ 1 - \eta f_{t}(x) + \eta^2 f_{t}(x)^2 }
	&&\text{(using $e^{z} \le 1+z+z^2$ for $z \le 1$)}
\\
&\leq 
	1 + \gamma - \eta \sum_{x \in [N]} q_{t}(x) f_{t}(x) 
	+ \eta^2 \sum_{x \in [N]} q_{t}(x) f_{t}(x)^2
~.
\end{align*}
Taking logarithms, using $\log(1+z) \le z$ for all $z > -1$, and summing over $t =
t_0,\ldots,t_1$ yields
$$
	\log \frac{W_{{t_1}+1}}{W_{t_0}} 
\leq
	\gamma T
	+ \eta^ 2 T
	- \eta \sum_{t=t_0}^{t_1} \sum_{x \in [N]} q_{t}(x) f_{t}(x)
	+ \eta^2 \sum_{t=t_0}^{t_1} \sum_{x \in [N]} q_{t}(x) f_{t}(x)^2
~.
$$
Moreover, since for all $t$ and $x$ we have $w_{t+1}(x) \ge w_t(x) \exp(-\eta f_t(x))$, for any fixed action $x^*$ we also have
\begin{align*}
	w_{t_1+1}(x^*)
\geq
	w_{t_0+1}(x^*) \exp\lr{-\eta \sum_{\mathclap{t=t_0+1}}^{t_1} f_t(x^*)}
\geq
	w_{t_0+1}(x^*) \exp\lr{-\eta \sum_{t=t_0}^{t_1} f_t(x^*)}
~,
\end{align*}
and since $W_{t_1+1} \ge w_{t_1+1}(x^*)$ and $w_{t_0+1}(x^*) \ge (\gamma/N) W_{t_0}$, we obtain
$$
	\log \frac{W_{t_1+1}}{W_{t_0}} 
\geq 
	-\eta \sum_{t=t_0}^{t_1}  f_{t}(x^*) - \log\frac{N}{\gamma}
~.
$$
Putting together and rearranging gives
\begin{align*}
	\sum_{t=t_0}^{t_1} \sum_{x \in [N]} q_{t}(x) f_{t}(x) 
	&\,-\, \sum_{t=t_0}^{t_1} f_{t}(x^*) 
\leq
	\frac{1}{\eta} \log\frac{N}{\gamma} 
	+ \frac{\gamma}{\eta} T
	+ \eta \sum_{t=t_0}^{t_1} \sum_{x \in [N]} q_{t}(x) f_{t}(x)^2
~.
\end{align*}
Finally, taking expectations and setting $\gamma = \tfrac{1}{T}$ yields the result.
\end{proof}

\subsubsection{\mwb: Amortized Mixed MW}

We now give an amortized version of the \mwa algorithm. 
Specifically, we give a variant of the latter algorithm that runs in $\tO(1)$ per round and attains an $\tO(\sqrt{NT})$ bound over the expected regret, as opposed to the \mwa algorithm that runs in time $\O(N)$ per round and achieves $\tO(\sqrt{T})$ regret.
The algorithm, which we call \mwb, is based on the \mwa update rule and incorporates sampling for accelerating the updates.%
\footnote{This technique is reminiscent of bandit algorithms; however, notice that here we separate exploration and exploitation: we sample two experts on each round, instead of one as required in the bandit setting.}

\begin{algorithm}[h]

\wrapalgo{
{\bf Parameters:} $\eta$, $\gamma$
\begin{enumerate}[nosep,leftmargin=0.5cm]
\item Initialize $w_1(x) = 1$ for all $x \in [N]$  
\item
For $t=1,2,\ldots$:
\begin{enumerate}[nosep]
	\item For all $x \in [N]$ compute $q_t(x) = w_t(x)/W_t$ with $W_t = \sum_y w_t(y)$
	\item Pick $x_t \sim q_t$, play $x_t$ and receive feedback $f_t$
	\item Pick $y_t \in [N]$ uniformly at random, and for all $x \in [N]$ update:
	\begin{align*}
		w_{t+1}(x) 
	\eq 
		\mycases
			{w_t(x) e^{-\eta N f_t(x)} + \tfrac{\gamma}{N} W_t}{if $x = y_t$}
			{w_t(x) + \tfrac{\gamma}{N} W_t}{otherwise}
	\end{align*}
\end{enumerate}
\end{enumerate}
}

\vspace{-0.5cm}
\caption{The \mwb algorithm.}
\label{alg:exp3}

\end{algorithm}

For \cref{alg:exp3} we prove:

\begin{lemma} \label{lem:exp3}
For any sequence of loss functions $f_{1},f_{2},\ldots : [N] \mapsto [0,1]$, and for $\gamma = \tfrac{1}{T}$ and any $\eta > 0$, the Algorithm \ref{alg:exp3} guarantees that
\begin{align*}
	\EE{ \sum_{t=t_0}^{t_1} f_{t}(x_t)  }
	&\,-\, \min_{x \in [N]} \sum_{t=t_0}^{t_1} f_{t}(x) 
\leq
	\frac{4\log(NT)}{\eta} + \eta NT
\end{align*}
in any time interval $\set{t_0,\ldots,t_1}$ of length at most $T$.  
Furthermore, the algorithm can be implemented to run in $O(\log{N})$ time per round.
\end{lemma}

\begin{proof}
We first derive the claimed regret bound as a simple consequence of \cref{lem:mw}.
Define a sequence of loss functions $\wh{f}_1,\ldots,\wh{f}_T$, as follows:
\begin{align*}
	\forall \, x \in [N]~,
\qquad
	\wh{f}_t(x)
\eq
	N f_t(y_t) \cdot \ind{x = y_t}
~.
\end{align*}
Notice that \cref{alg:exp3} is essentially applying \mwa updates to the loss functions $\wh{f}_1,\ldots,\wh{f}_T$ instead of to the original ones. 
Thus, we obtain from \cref{lem:mw} that
\begin{align*}
	\EE{ \sum_{t=t_0}^{t_1} \wh{f}_{t}(x_t) } 
	\,-\, \EE{ \sum_{t=t_0}^{t_1} \wh{f}_{t}(x^*) }
\leq
	\frac{4\log(NT)}{\eta} 
	+ \eta \EE{ \sum_{t=t_0}^{t_1} \wh{f}_{t}(x_t)^2 }
\end{align*}
for any fixed $x^* \in [N]$.
We now get the lemma by observing that $\E[\wh{f}_t(x)] = f_t(x)$ and $\E[\wh{f}_t(x)^2] = N f_t(x) \le N$ for all $t$ and $x$ (also notice that $x_t$ is independent of $\wh{f}_t$).

It remains to prove that the algorithm's updates can be carried out in time $O(\log N)$ per round.
The weights $w_t(x)$ can be maintained implicitly as a sum of variables  $w_t(x) = \alpha_t(x) + \beta_t $, where $\beta_t$ captures the amount of uniform distribution for the $t$'th weights.
The main observation is that $w_t(x)$ can now be updated via:
\begin{align*}
	\alpha_{t+1}(x ) 
&\eq
	(\alpha_{t}(x) + \beta_{t} ) e^{-\eta \wh{f}_t(x) } - \beta_{t}
~,
\\
	\beta_{t+1}
&\eq
	\beta_{t} + \frac{\gamma}{N} \lr{ \tsum_x \alpha_{t}(x) + N \beta_{t} }
~.
\end{align*}
Notice that the vector $\alpha_{t+1}$ has only one component that needs updating per iteration---the one corresponding to $y_t$. 
The update of the scalar $\beta_{t+1}$ needs the sum of all parameters $\alpha_t(x)$, that can be maintained efficiently alongside the individual weights.

Finally, we explain how to sample efficiently from $q_t$. 
We can write
\begin{align*}
	q_t(x)
\eq
	\mu_t \cdot \frac{1}{N} + (1-\mu_t) \cdot \frac{\alpha_t(x)}{\sum_y \alpha_t(y)}
~,
\end{align*}
with $\mu_t = \beta_t/(\beta_t + \tfrac{1}{N}\sum_y \alpha_t(y))$.
Thus, sampling from $q_t$ can be carried out in two stages: with probability $\mu_t$ sample uniformly at random; with the remaining probability, sample according to the weights $\alpha_t(x)$.
In order to implement the latter sampling operation in time $O(\log{N})$, we can maintain a binary tree with $N$ leaves that correspond to the weights $\alpha_t(x)$, and with the internal nodes caching the total weight of their descendants.
\end{proof}

\subsubsection{\mwc: Sliding Amortized Mixed MW}

The final component we require is a version of \mwb that works in an online learning setting with \emph{activated experts}.
In this version, on each round of the game one of the experts is ``activated''. 
The sequence of activations is determined based only on the loss values and does not depend on past decisions of the algorithm; thus, it can be thought of as set by the oblivious adversary before the game starts.
The goal of the player is to compete only with the last $k$ (distinct) activated experts, for some parameter $k$.
In the context of the present section, the expert activated on round $t$ is the leader at the end of that round. 
Therefore, we overload notation and denote by $x^*_t$ the expert activated on round $t$.

\begin{algorithm}[h]

\wrapalgo{
{\bf Parameters:} $k$, $\eta$, $\gamma$
\begin{enumerate}[nosep,leftmargin=0.5cm]
\item Initialize $B_1(i) = i$ and $w_1(i) = 1$ for all $i \in [k]$
\item
For $t=1,2,\ldots$:
\begin{enumerate}[nosep]
	\item For all $i \in [k]$ compute $q_t(i) = w_t(i)/W_t$ with $W_t = \sum_y w_t(i)$
	\item Pick $i_t \sim q_t$, play $x_t = B_t(i_t)$, receive $f_t$ and new activated expert $x^*_t$
	\item Update weights: pick $j_t \in [k]$ uniformly at random, and set:
	\begin{align*}
		\forall\, i \in [k]~,
	\qquad
		w_{t+1}(i) 
	\eq 
		\mycases
			{w_t(i) e^{-\eta N f_t(B_t(i))} + \tfrac{\gamma}{N} W_t}{if $i = j_t$}
			{w_t(i) + \tfrac{\gamma}{N} W_t}{otherwise}
	\end{align*}
	\item Update buffer: set $B_{t+1} = B_t$; if $x^*_t \notin B_t$, find the index $i' \in [k]$ of the oldest activated expert in $B_{t}$ (break ties arbitrarily) and set $B_{t+1}(i') = x^*_t$.
\end{enumerate}
\end{enumerate}
}

\vspace{-0.5cm}
\caption{The \mwc algorithm.}
\label{alg:exp3q}

\end{algorithm}

The \mwc algorithm, presented in \cref{alg:exp3q}, is a ``sliding window'' version of \mwb that keeps a buffer of the last $k$ (distinct) activated experts. When its buffer gets full and a new expert is activated, the algorithm evicts from the buffer the expert whose most recent activation is the oldest. 
(Notice that the latter expert is not necessarily the oldest one in the buffer, as an expert can be re-activated while already in the buffer.)
In this case, the newly inserted expert is assigned \emph{the same weight} of the expert evicted from the buffer.

For the \mwc algorithm we prove:

\begin{lemma} \label{lem:exp3q}
Assume that expert $x^* \in [N]$ was activated on round $t_0$, and from that point until round $t_1$ there were no more than $k$ different activated experts (including $x^*$ itself).
Then, for any time interval $[t_0',t_1'] \subseteq [t_0,t_1]$ of length at most $T$, the $\mwc$ algorithm with parameters $k$, $\gamma = \tfrac{1}{T}$ and any $\eta > 0$ guarantees
\begin{align*}
	\EE{ \sum_{t=t_0'+1}^{t_1'} f_t(x_t) } 
	\,-\, \sum_{t=t_0'+1}^{t_1'} f_t(x^*)
\leq
	\frac{4\log(kT)}{\eta} + \eta kT
~.
\end{align*}
Furthermore, the algorithm can be implemented to run in time $\tO(1)$ per round.
\end{lemma}

\begin{proof}
For each index $i=1,\ldots,k$, imagine a ``meta-expert'', whose loss on round $t$ of the game is equal to the loss incurred on the same round by the expert occupying entry $i$ of the buffer $B_t$ on round $t$. 
Then, notice that we can think of the algorithm as following \mwb updates over these meta-experts. 
Importantly, the assignment of losses to the meta-experts is oblivious to the predictions of the algorithm: indeed, the only factor affecting the addition/removal of experts to/from the buffer is the pattern of their activation, which is being decided by the adversary before the game begins.

Furthermore, as long as an expert is not evicted from the buffer during a certain period, it occupies the same entry and thus is associated with the same meta-expert throughout that period.
In particular, throughout the period between rounds $t_0'+1$ and $t_1'$, the expert $x^*_t$ is associated with some fixed meta-expert: $x^*_t$ was activated on round $t_0 \le t_0'$, and was not removed from the buffer by round $t_1'$ because there were at most $k$ different activated experts up to round $t_1 \ge t_1'$.
Hence, the first claim of the lemma follows directly from the regret bound of \mwb stated in \cref{lem:exp3}.

Finally, the claim regarding the runtime of the algorithm can be obtained by using standard data structures for the maintenance of the buffer (and the associated weights). 
For example, $B_t$ can be maintained sorted according to last time of activation;
for accommodating the membership query $x^*_t \notin B_t$ in $\tO(1)$ time, one can maintain an additional copy of $B_t$ represented by a set data structure. 
\end{proof}

\section{Solving Zero-sum Games with Best-response Oracles}
\label{sec:games}

In this section we apply our online algorithms to repeated game playing in zero-sum games with best-response oracles.
Before we do that, we first have to extend our results to the case where the assignment of losses can be adaptive to the decisions of the algorithm.
Namely, unlike in the standard model of an oblivious adversary where the loss functions are being determined before the game begins, in the adaptive setting the loss function $f_t$ on round $t$ may depend on the (possibly randomized) decisions $x_1,\ldots,x_{t-1}$ chosen in previous rounds.

Fortunately, with minor modifications \cref{alg:N14} can be adapted to the non-oblivious setting and obtain low regret against adaptive adversaries as well.
In a nutshell, we show that the algorithm can be made ``self-oblivious'', in the sense that its decision on each round depends only \emph{indirectly} on its previous decisions, through its dependence on the previous loss functions; algorithms with this property are ensured to work well against adaptive adversaries \citep[see][]{mcmahan2004online,dani2006robbing,CesaBianchiLugosi06book}.
Furthermore, the same property is also sufficient for the adapted algorithm to obtain low regret with high probability, and not only in expectation.
The formal details are given in the proof of the following corollary.

\begin{corollary} \label{cor:N14-hp}
With probability at least $1-\del$, the average regret of \cref{alg:N14} (when implemented appropriately) is upper bounded by $40 N^{1/4}\log(\tfrac{NT}{\del})/\sqrt{T}$.
This is true even when the algorithm faces a non-oblivious adversary.
\end{corollary}

\begin{proof}
We first explain how \cref{alg:N14} can be implemented so as the following ``self-obliviousness'' property holds, for all rounds $t$:
\begin{align} \label{eq:self-oblv}
	\Pr(x_t = x \mid x_1,\ldots,x_{t-1}, f_1,\ldots,f_{t-1})
\eq
	\Pr(x_t = x \mid f_1,\ldots,f_{t-1})
~.
\end{align}
We can ensure this holds by randomizing separately for making the decisions $x_t$, and for updating the algorithm.
That is, if $x_t$ is sampled from a distribution $p_t$ on round $t$ of the game, then we discard $x_t$ and use a different independent sample $x_t' \sim p_t$ for updating the algorithm. In fact, \cref{alg:N14} makes multiple updates on each round (to the various $\cA_{r,s}$ algorithms, etc.); we ensure that the sample $x_t$ is not used in any of these updates, and a fresh sample is picked when necessary.
Further, we note that this slight modification does not impact the runtime of the algorithm, up to constants.

With \cref{alg:N14} implemented this way, we can now use, e.g., Lemma~4.1 of \citet{CesaBianchiLugosi06book} to obtain from \cref{thm:online-upper} that
\begin{align*}
	\frac{1}{T} \sum_{t=1}^{T} f_t(x_t) 
	\,-\, \min_{x \in [N]} \frac{1}{T} \sum_{t=1}^{T} f_t(x)
\leq
	\frac{40 N^{1/4}\log(NT)}{\sqrt{T}} + \sqrt{\frac{\log\tfrac{1}{\del}}{2T}}
\end{align*}
holds with probability at least $1-\del$ against any non-oblivious adversary.
Further upper bounding the right-hand side of the above yields the stated regret bound.
\end{proof}

Using standard techniques \citep{freund1999adaptive}, we can now use \cref{alg:N14} to solve zero-sum games equipped with best-response oracles.
The simple scheme is presented in \cref{alg:zsg}: both players use the online \cref{alg:N14} to produce their decisions throughout the game, and employ their best response oracles to compute the ``leaders''. In the context of zero-sum games, the leader on iteration $t$ is the best response to the empirical distribution of the past plays of the opponent.

\begin{algorithm}[h]

\wrapalgo{
{\bf Parameters:} game matrix $G \in [0,1]^{N \times N}$, parameter $T$
\begin{enumerate}[nosep,leftmargin=0.5cm]
\item Initialize instances $\cA_1,\cA_2$ of \cref{alg:N14} with parameters $N,T$ for the row and column players, respectively
\item
For $t=1,2,\ldots,T$:
\begin{enumerate}[nosep]
	\item Let the players play the decisions $x_t, y_t$ of $\cA_1,\cA_2$, respectively
	\item Let $\bar{p}_t$ be the empirical distribution of $x_1,\ldots,x_t$, and  
let $\bar{q}_t$ be the empirical distribution of $y_1,\ldots,y_t$
	\item Update $\cA_1$ with the loss function $G(\cdot,y_t)$ and the leader $x^*_t = \br{1}(\bar{q}_t)$
	\item Update $\cA_2$ with the loss function $G(x_t,\cdot)$ and the leader $y^*_t = \br{2}(\bar{p}_t)$
\end{enumerate}
\end{enumerate}
{\bf Output:} the profile $(\bar{p},\bar{q}) = (\bar{p}_T,\bar{q}_T)$
}

\vspace{-0.5cm}
\caption{Algorithm for zero-sum games with best-response oracles.}
\label{alg:zsg}

\end{algorithm}

We remark that, in fact, it is sufficient that only one of the players follow \cref{alg:N14} for ensuring fast convergence to equilibrium. The other player could, for example, behave greedily and simply follow his best response to the plays of the first player---a strategy known as ``fictitious play''.

Finally, we analyze \cref{alg:zsg}, thereby proving \cref{thm:games-upper}.

\begin{repthm}[restated]{thm:games-upper}
With probability at least $1-\del$, \cref{alg:zsg} with $T =  \tfrac{240^2 \sqrt{N}}{\eps^2} \log^2\tfrac{240 N}{\eps \del}$ outputs a mixed strategy profile $(\bar{p},\bar{q})$ being an $\eps$-approximate equilibrium.
The algorithm can be implemented to run in $\tO(\sqrt{N}/\eps^2)$ time.
\end{repthm}

\begin{proof}
First, we note the running time. 
Notice that the empirical distributions $\bar{p}_t$ and $\bar{q}_t$ do not have to be recomputed each round, and can be maintained incrementally with constant time per iteration.
Furthermore, since we assume that each call to the best response oracles costs $\O(1)$ time, \cref{thm:online-upper} shows that the updates of $\cA_1,\cA_2$ can be implemented in time $\tO(1)$. 
Hence, each iteration costs $\tO(1)$ and so the overall runtime is $\tO(\sqrt{N}/\eps^2)$.

We move on to analyze the output of the algorithm.
Using the regret guarantee of \cref{cor:N14-hp} for the online algorithm of the row player, we have
\begin{align*} 
	\frac{1}{T} \sum_{t=1}^T G(x_t,y_t)
	\,-\,  \min_{x \in [N]} \frac{1}{T} \sum_{t=1}^{T} G(x,y_t)
\leq 
	\frac{40 N^{1/4}}{\sqrt{T}} \log\frac{2NT}{\del}
\end{align*}
with probability at least $1-\tfrac{\del}{2}$.
Similarly, for the online algorithm of the column player we have
\begin{align*} 
	\max_{y \in [N]} \frac{1}{T} \sum_{t=1}^{T} G(x_t,y)
	\,-\,  \frac{1}{T} \sum_{t=1}^T G(x_t,y_t)
\leq 
	\frac{40 N^{1/4}}{\sqrt{T}} \log\frac{2NT}{\del}
\end{align*}
with probability at least $1-\tfrac{\del}{2}$.
Hence, summing the two inequalities and using $\frac{1}{T} \sum_{t=1}^{T} G(x_t,y) = G(\bar{p},y)$ and $\frac{1}{T} \sum_{t=1}^{T} G(x,y_t) = G(x,\bar{q})$, we have
\begin{align} \label{eq:zsg-maxmin}
	\max_{q \in \Del_N} G(\bar{p},q)
	\,-\,  \min_{p \in \Del_N} G(p,\bar{q})
\leq 
	\frac{80 N^{1/4}}{\sqrt{T}} \log\frac{NT}{\del}
\leq
	\eps
\end{align}
with probability at least $1-\del$; the ultimate inequality involves our choice of $T$ and a tedious calculation.

Now, let $(p^*,q^*)$ be an equilibrium of the game, and denote by $\lambda^* = G(p^*,q^*)$ the value of the game.
As a result of \cref{eq:zsg-maxmin}, for all $q \in \Del_N$ we have
\begin{align*}
	G(\bar{p},q)
\leq
	\min_{p \in \Del_N} G(p,\bar{q}) + \eps
\leq
	G(p^*,\bar{q}) + \eps
\leq
	G(p^*,q^*) + \eps
\eq
	\lambda^* + \eps
~.
\end{align*}
Similarly, for all $p \in \Del_N$,
\begin{align*}
	G(p,\bar{q})
\geq
	\max_{q \in \Del_N} G(\bar{p},q) - \eps
\geq
	G(\bar{p},q^*) - \eps
\geq
	G(p^*,q^*) - \eps
\eq
	\lambda^* - \eps
~.
\end{align*}
This means that, with probability at least $1-\del$, the mixed strategy profile $(\bar{p},\bar{q})$ is an $\eps$-approximate equilibrium.
\end{proof}

\section{Lower Bounds}
\label{sec:lower}

In this section we prove our computational lower bounds for optimizable experts and learning in games, stated in \cref{thm:online-lower,thm:games-lower}.
We begin with the latter, and then prove \cref{thm:online-lower} via reduction.

%


\begin{repthm}[restated]{thm:games-lower}
For any efficient (randomized) players in a repeated $N \times N$ zero-sum game with best-response oracles, there exists a game matrix with $[0,1]$ values such that with probability at least~$\frac{2}{3}$, the average payoff of the players is at least~$\frac{1}{4}$ far from equilibrium after~$\Omega(\sqrt{N}/\log^{3}{N})$ time. 
\end{repthm}

We prove \cref{thm:games-lower} by means of a reduction from a local-search problem studied by \citet{aldous1983minimization}, which we now describe.

\paragraph{Lower bounds for local search.}

Consider a function $f: \{0,1\}^d \mapsto \mathbb{N}$ over the $d$-dimensional hypercube. 
A local optimum of $f$ over the hypercube is a vertex such that the value of $f$ at this vertex is larger than or equal to the values of $f$ at all neighboring vertices (i.e., those with hamming distance one). 
A function is said to be \emph{globally-consistent} if it has a single local optimum (or in other words, if every local optimum is also a global optimum of the function). 

\citet{aldous1983minimization} considered the following problem (slightly rephrased here for convenience), which we refer to as \emph{Aldous' problem}:
Given a globally-consistent function $f: \{0,1\}^d \mapsto \mathbb{N}$ (given as a black-box oracle), determine whether the maximal value of $f$ is even or odd with a minimal number of queries to the function.

The following theorem is an improvement of \citet{aaronson2006lower} to a result initially proved by \citet{aldous1983minimization}.

\begin{theorem}[\citealp{aldous1983minimization,aaronson2006lower}] \label{thm:aldous}
For any randomized algorithm for Aldous' problem that makes no more than $\Omega(2^{d/2}/d^{2})$ value queries in the worst case, there exists a function $f : \set{0,1}^{d} \mapsto \mathbb{N}$ such that the algorithm 
cannot determine with probability higher than~$\frac{2}{3}$ whether the maximal value of $f$ over $\set{0,1}^{d}$ is even or odd.
\end{theorem}


We reduce Aldous' problem to the problem of approximately solving a large zero-sum game with best-response and value oracles. 
Our reduction proceeds by constructing a specific form of a game, which we now describe. 

\paragraph{The reduction.}

Let $f : [N] \mapsto \mathbb{N}$ be an input to Aldous' problem, with maximal value $f^{*} = \max_{i \in [N]} f(i)$. (Here, we identify each vertex of the $\ceil{\log_{2}{N}}$-dimensional hypercube with a natural number in the range $1,\ldots,N$ corresponding to its binary representation.)
We shall describe a zero-sum game with value $\lambda = \lambda(f^{*})$, where  
\begin{align*}
	\forall ~ k \in \mathbb{N} ~,
	\qquad
	\lambda(k) 
	\eq \mycases
		{\frac{1}{4}}{if $k$ is even,}
		{\frac{3}{4}}{if $k$ is odd.}
\end{align*}
%
Henceforth, we use $\Gamma(V)$ to denote the set of neighbors of a set of vertices $V \subseteq [N]$ of the hypercube (that includes the vertices in $V$ 
themselves).

The game matrix and its corresponding oracles are constructed as follows:

\begin{itemize}
\item \emph{Game matrix:}
Based on the function $f$, define a matrix $G^{f} \in [0,1]^{N \times N}$ as follows:
$$ 
	\forall ~ i,j \in [N] ~, 
	\qquad 
	G^{f}_{i,j} \eq \mythreecases 
		{\lambda(f(i))}{if $i$ and $j$ are local maxima of $f$,}
		{0} {if $f(i) \ge f(j)$,}
		{1}{otherwise.} 
$$

\item \emph{Value oracle:}
The oracle $\Oval(i,j)$ simply returns $G^{f}_{i,j}$ for any $i,j \in [N]$ given as input.

\item \emph{Best-response oracles:}  
For any mixed strategy $p \in \Del_{N}$, define:
$$ 
	\br{1}(p) 
	\eq \br{2}(p) 
	\eq \argmax_{\mathclap{i \,\in\, \Gamma(\supp(p))}} ~ f(i) ~ .
$$
\end{itemize}

%

We set to prove \cref{thm:games-lower}.
First, we assert that the value of the described game indeed equals~$\lambda$.

\begin{lemma}
The minimax value of the game described by the matrix $G^{f}$ equals $\lambda$.  
\end{lemma}

\begin{proof}
Let $i^{*} = \argmax_{i \in [N]} f(i)$ be the global maxima of $f$.
We claim that the (pure) strategy profile $(i^{*},i^{*})$ is an equilibrium of the game given by $G^{f}$.
To see this, notice that the payoff with this profile is $\lambda(f(i^{*})) = \lambda$.
However, any profile of the form $(i,i^{*})$ generates a payoff of either $G^{f}_{i,i^{*}} = \lambda$ or $G^{f}_{i,i^{*}} = 1$, since $f(i) \le f(i^{*})$ for each $i$. Hence, the row player does not benefit by deviating from playing $i^{*}$.
Similarly, a profile of the form $(i^{*},j)$ yields a payoff of at most $0$, thus a deviation from $i^{*}$ is not profitable to the column player either. 
%
\end{proof}

Next, we show that the value and best-response oracles we specified are correct. 

\begin{lemma} \label{lem:games-oracles}
The procedures $\Oval$ and $\br{1},\br{2}$ are correct value and best-response oracles for the game $G^{f}$.
\end{lemma}

\begin{proof}
The oracle $\Oval$ is trivially a valid value oracle for the described game, as $\Oval(i,j) = \smash{G^{f}_{i,j}}$ for all $i,j \in [N]$ by definition.
Moving on to the best-response oracles, we shall prove the claim for the oracle $\br{1}$; the proof for $\br{2}$ is very similar and thus omitted. 

Consider some input $p \in \Del_{N}$ and denote $j = \br{1}(p)$ the output of the oracle.
Note that if $i^{*} \in \supp(p)$, then $j = \argmax\set{f(i) \,:\, i \,\in\, \Gamma(\supp(p))} = i^{*}$ as $f$ has a single global maxima at $i^{*}$. 
The main observation in this case is that the strategy $j=i^{*}$ dominates all other column strategies $j'$: indeed, it is not hard to verify that since $f(j) \ge f(j')$ for any $j'$, we have $\smash{ G^{f}_{i,j} \ge G^{f}_{i,j'} }$ for all $j'$.
This implies that $p\tr G^{f} e_{j} \ge p\tr G^{f} e_{j'}$ for all $j'$, namely, $j$ is a best-response to $p$.

On the other hand, if $i^{*} \notin \supp(p)$ we claim that $p\tr G^{f} e_{j} = 1$, which would immediately give $p\tr G^{f} e_{j} \ge p\tr G^{f} e_{j'}$ for all $j'$ (as the maximal payoff in the game is $1$).
This follows because if $i^{*} \notin \supp(p)$, then it must hold that $f(j) > f(i)$ for all $i \in \supp(p)$, which means that $G^{f}_{i,j} = 1$ for all $i \in \supp(p)$. 
Hence, $p\tr G^{f} e_{j} = 1$ as claimed.
\end{proof}	

We can now prove our main theorem.

\begin{proof}[Proof of \cref{thm:games-lower}] 
Let $f : [N] \mapsto \mathbb{N}$ be an arbitrary globally-consistent function over the hypercube.
Consider some algorithm that computes the value of the zero-sum game $G^{f}$ up to an additive error of $\frac{1}{4}$ with probability at least $\frac{2}{3}$. 
Determining the game value up to $\frac{1}{4}$ gives the value of $\lambda$ (that can equal either $\frac{1}{4}$ or $\frac{3}{4}$), which by construction is determined according to the maximal value of $f$ being even or odd. 
Thus, the number of queries to $f$ the algorithm makes, through one of the oracles, is lower-bounded by $\Omega(\sqrt{N}/\log^{2}{N})$ according to \cref{thm:aldous}.
In what follows, we show how this lower bound can be translated to a lower bound on the runtime of the algorithm.


First, notice that the runtime of the algorithm is lower bounded by the total number of row/column indices touched by the algorithm, namely, the total number of indices that appear in inputs to one of the oracles $\Oval$, $\Oopt$ at some point throughout the execution of the algorithm (we think of index $i$ as appearing in the input of the call $\Oopt(p)$ if $i \in \supp(p)$).
Hence, if we let~$S$ denote the set of all indices touched by the algorithm, then it is enough to lower bound the size of~$S$ in order to obtain a lower bound on the runtime of the algorithm.

Now, notice that the set of all entries of $f$ queried by the algorithm throughout its execution, via one of the oracles $\Oval$ and $\Oopt$, is a subset of $\cup_{i \in S} \Gamma(i)$.
Indeed, upon any index $i \in S$ that appears in the input to one of the oracles, the function $f$ has to be queried only at the neighborhood~$\Gamma(i)$ to produce the required output (recall the definitions of the oracles in our construction of $G_f$ above).
Hence, the total number of distinct queries to $f$ made by the algorithm is at most $\O(\abs{S} \cdot \log{N})$.
On the other hand, as we noted earlier, this number is lower bounded by $\Omega(\sqrt{N}/\log^{2}{N})$ as a result of \cref{thm:aldous}. 
Both bounds together yield the lower bound $\abs{S} = \Omega(\sqrt{N}/\log^{3}{N})$, which directly implies the desired runtime lower bound.
\end{proof}

Now, we can obtain \cref{thm:online-lower} as a direct corollary of the lower bound for zero-sum games. 
We remark that it is possible to prove a tighter lower bound than the one we prove here, via direct information-theoretic arguments; we defer details to \cref{sec:online-lower}.

\begin{repthm}[restated]{thm:online-lower}
Any (randomized) algorithm in the optimizable experts model cannot guarantee an expected average regret smaller than $\tfrac{1}{16}$ over at least $T \ge 20$ rounds in total time better than $\O(\sqrt{N}/\log^3{N})$.
\end{repthm}

\begin{proof}
Suppose that there exists an online algorithm in the optimizable experts model that guarantees expected average regret $< \frac{1}{16}$ in total time $\tau$, where $\tau = o(\sqrt{N}/\log^3{N})$.
Following the line of arguments presented in \cref{sec:games}, we can show that the algorithm can be used to approximate zero-sum games.
First, as explained in \cref{sec:games}, we may assume that the algorithm is self-oblivious (see that section for the definition), and for such an algorithm Lemma~4.1 of \cite{CesaBianchiLugosi06book} shows that  with probability at least $1-\del = \frac{5}{6}$, the average regret after $T \ge 20$ rounds is at most $\frac{1}{16} + \smash{\sqrt{\log(1/\del)/2T}} \le \frac{1}{8}$.
Then, following the proof of \cref{thm:games-upper} we can show that the online algorithm, if deployed by two players in a zero-sum game, can be used to approximate the equilibrium of the game to within $\frac{1}{4}$ with probability at least $\frac{2}{3}$ and, with access to best response oracles, in total runtime~$\O(\tau)$.
This is a contradiction to the statement of \cref{thm:games-lower}, proving our claim.
\end{proof}

\subsection{Online Lipschitz-Continuous Optimization}
\label{sub:lipschitz}

In this section we present a simple consequence of \cref{thm:online-lower}: we show that any reduction from online Lipschitz-continuous optimization in $d$-dimensional Euclidean space to the corresponding offline problem, must run in time exponential in the dimension $d$. 
Notice that we do not assume convexity of the loss functions: for convex (and Lipschitz) functions it is well known that dimension-free regret bounds are possible~\citep{zinkevich2003online}.

The online optimization model with oracles is very similar to the model we presented in \cref{sec:model}.
The only difference is that in online optimization, the decision set $\X$ is a compact subset of a $d$-dimensional Euclidean space. 
The main result of this section shows that even when the functions $f_1,f_2,\ldots$ are all $1$-Lipschitz, the runtime required for convergence of the average regret in this model is exponential in the dimension $d$.

\begin{corollary} \label{cor:lip-hard}
For any (randomized) algorithm in the oracle-based online optimization model, there are oracles $\Oval$ and $\Oopt$ and a sequence $f_{1},f_{2},\ldots : [0,1]^d \mapsto [0,1]$ of $1$-Lipschitz loss functions in $d$ dimensions such that the runtime required for the algorithm to attain expected average regret smaller than $\half$ is $\Omega(2^{d/2}/\poly(d))$.
\end{corollary}

We prove the corollary via reduction from \cref{thm:online-lower}. 
The idea is to embed a discrete online problem over $N$ experts in $d = \ceil{\log N}$ dimensions. 
To this end, we view functions over the hypercube $\set{0,1}^{d}$ as functions over the set of experts $\H = [N]$ by identifying expert $i$ with the vertex corresponding to $i$'s binary representation using $d$ bits.
%
%
Given a function~$f : [N] \mapsto [0,1]$, we define a function $\wt{f}$ over $\K = [0,1]^{d}$ by
\begin{align*}
	\forall ~ x \in [0,1]^{d} ~,
	\qquad
	\wt{f}(x) \eq
	\frac{1}{\sqrt{d}} \, \sum_{\xi \in \set{0,1}^{d}} f(\xi) 
		\cdot \prod_{\xi_{i}=1}  x_{i} \prod_{\xi_{i}=0} (1-x_{i})
	~.
\end{align*}
It is not hard to show that $\wt{f}$ is Lipschitz-continuous over $[0,1]^{d}$.
\begin{lemma}
The function $\wt{f}$ is $1$-Lipschitz over $[0,1]^{d}$ (with respect to the Euclidean norm).
\end{lemma}

\begin{proof}
Notice that $\wt{f}$ is linear in $x_{i}$, thus for all $x \in [0,1]^{d}$,
\begin{align*}
	\abs{\partial_{i} \wt{f} (x)}
&\eq 
	\Big| \frac{\partial \wt{f}}{\partial x_{i}} (x_{1},\ldots,x_{d}) \Big| 
\\
&\eq 
	\big| 
		\wt{f}(x_{1},\ldots,x_{i-1},1,x_{i+1},\ldots,x_{d})
		- \wt{f}(x_{1},\ldots,x_{i-1},0,x_{i+1},\ldots,x_{d}) \big| 
\\
&\leq 
	\frac{1}{\sqrt{d}} ~.
\end{align*}
Hence, $\norm{\nabla \wt{f}(x)}_{2}^{2} = \sum_{i=1}^{d} \abs{\partial_{i} \wt{f} (x)}^{2} \le 1$ which implies that $\wt{f}$ is $1$-Lipschitz.
\end{proof}

The following lemma shows that an optimization oracle over $[N]$ 
can be directly converted into an oracle for the extensions over the convex set $[0,1]^{d}$.

\begin{lemma} \label{lem:lip-opt}
For any functions $f_{1},\ldots,f_{m} : \set{0,1}^d \mapsto [0,1]$ and scalars $\alpha_{1},\ldots,\alpha_{m} \ge 0$ we have
\begin{align*}
	\min_{x \in [0,1]^{d}} \sum_{i=1}^{m} \alpha_{i} \wt{f}_{i}(x)
\eq 
	\frac{1}{\sqrt{d}} \, \min_{x \in \set{0,1}^{d}} \sum_{i=1}^{m} \alpha_{i} f_{i}(x) ~.
\end{align*}
\end{lemma}

\begin{proof}
Denote $f = \sum_{i=1}^{m} \alpha_{i} f_{i}(x)$ and notice that $\wt{f} = \sum_{i=1}^{m} \alpha_{i} \wt{f}_{i}(x)$.
Then, the lemma claims that 
$$
	\min_{x \in [0,1]^{d}} \wt{f}(x) 
	\eq \frac{1}{\sqrt{d}} \, \min_{x \in \set{0,1}^{d}} f(x)
	~.
$$
This follows directly from the definition of~$\wt{f}$, as $\wt{f}(x)$ is a convex combination of the values of $f$ on the vertices of the hypercube, thus the minimum of $\wt{f}$ must be attained at one of the  vertices.
\end{proof}

\begin{lemma} \label{lem:lip-reduction}
Let $f_{1},\ldots,f_{T} : \set{0,1}^d \mapsto [0,1]$ be a sequence of functions and let $\wt{f}_{1},\ldots,\wt{f}_{T} : [0,1]^{d} \mapsto [0,1]$ be the corresponding Lipschitz extensions.
Given an algorithm that achieves regret $R_{T}$ on $\wt{f}_{1},\ldots,\wt{f}_{T}$ over the decision set $[0,1]^{d}$, one can efficiently achieve expected regret of $\sqrt{d} R_{T}$ on $f_{1},\ldots,f_{T}$ over the decision set $\set{0,1}^d$.
\end{lemma}

\begin{proof}
Assume that the algorithm produced $x_{1},\ldots,x_{T} \in [0,1]^{d}$ such that
\begin{align*}
	\sum_{t=1}^{T} \wt{f}_{t}(x_{t}) 
	- \min_{x^{*} \in [0,1]^{d}} \sum_{t=1}^{T} \wt{f}_{t}(x^{*})
\leq 
	R_{T}
~.
\end{align*}
Noticing, by \cref{lem:lip-opt} above, that
\begin{align*}
	\min_{x^{*} \in [0,1]^{d}} \sum_{t=1}^{T} \wt{f}_{t}(x^{*})
\eq 
	\frac{1}{\sqrt{d}} \, \min_{y^{*} \in \set{0,1}^{d}} 
		\sum_{t=1}^{T} f_{t}(y^{*})
~,
\end{align*}
and using randomized rounding to obtain points $y_{1},\ldots,y_{T} \in \set{0,1}^{d}$ such that $\E[f_{t}(y_{t})] = \sqrt{d} \wt{f}_{t}(x_{t})$ (that is, by choosing $y_{t}(i) = 1$ with probability $x_{t}(i)$ for each $i$ independently), we get
\begin{align*}
	\E\left[ \sum_{t=1}^{T} f_{t}(y_{t}) \right] 
	- \min_{y^{*} \in \set{0,1}^{d}} \sum_{t=1}^{T} f_{t}(y^{*})
\leq
	\sqrt{d} R_{T} 
~.
\end{align*}
Hence, the sequence of actions $y_{1},\ldots,y_{T} \in \set{0,1}^d$ achieves regret of $\sqrt{d} R_{T}$ in expectation with respect to the functions $f_{1},\ldots,f_{T}$.
\end{proof}

\cref{cor:lip-hard} now follows directly from \cref{lem:lip-reduction,thm:online-lower}.

\bibliographystyle{abbrvnat}
\bibliography{bib}
\appendix

\section{Tighter Lower Bound for Optimizable Experts} 
\label{sec:online-lower}

Here we prove a slightly tighter version of \cref{thm:online-lower} than the one we proved in \cref{sec:lower}.

\begin{repthm}[restated]{thm:online-lower}
Let $N>0$ and fix $\X = \Y = [N]$.
For any (randomized) regret minimization algorithm, there is a loss function~$\ell : \X \times \Y \mapsto [0,1]$, corresponding oracles $\Oval$, $\Oopt$, and a sequence of actions $y_{1},y_{2},\ldots \in \Y$ such that the runtime required for the algorithm to attain expected average regret smaller than $\half$ is at least $\Omega(\sqrt{N})$.
\end{repthm}

The proof proceeds by demonstrating a randomized construction of a hard online learning problem with optimizable experts, which we now describe.
For simplicity, we assume that $N = n^{2}$ for some integer $n>1$.
Pick a set $X^{*} = \set{x_{1}^{*},\ldots,x_{n}^{*}} \subseteq [N]$ of $n$ ``good'' experts, by choosing $x_{i}^{*} \in X_{i} = \set{n(i-1)+1,\ldots,ni}$ uniformly at random for each $i=1,\ldots,n$.
Then, define the loss function:
\begin{align*}
	\forall ~ x, y \in [N] ~,
	\qquad
	\ell(x,y) 
	\eq \mycases{0}{if $x,y \in X^{*}$ and $x \ge y$,}{1}{otherwise.}
\end{align*}

\cref{thm:online-lower} is obtained as a direct corollary of the following result. 

\begin{theorem} \label{thm:online-lb}
Consider an adversary that picks the sequence $y_{1},y_{2},\ldots,y_{n}$, where $y_{t} = x_t^*$ for all~$t$.
Any online algorithm whose expected runtime is $o(\sqrt{N})$ cannot attain expected average regret smaller than $\half$ (at some point $1 \le t \le n$ during the game) on the sequence~$y_{1},\ldots,y_{n}$.
\end{theorem}

For the proof, we make a simple observation given in the following lemma.

\begin{lemma} \label{lem:maxsupp1}
For any distribution $p \in \Del_N$, there exists an $x^* \in \supp(p)$ which is a valid answer to the query~$\Oopt(p)$, namely, such that $x^*$ is a row index of some best expert with respect to $p$.
\end{lemma}

\begin{proof}
A valid optimization oracle for the loss function $\ell$ defined above is given by
\begin{align*}
	\forall ~ p \in \Del_{N} ~,
	\qquad
	\Oopt(p) 
	\eq \mycases
		{\max \set{\supp(p) \cap X^{*}}} 
			{if $\supp(p) \cap X^{*} \ne \emptyset$,}
		{\max\set{\supp(p)}}
			{otherwise.}
\end{align*}
It is now seen that for this oracle, it holds that $\Oopt(p) \in \supp(p)$ for any $p \in \Del_N$.
\end{proof}

Before proving \cref{thm:online-lb}, we state and prove a lemma which is key to our analysis. 

\begin{lemma} \label{lem:folklore}
Let $A$ be an array of size $n$ formed by choosing an entry uniformly at random and setting its value to $x \ne 0$, while keeping all other entries set to zero.
Any algorithm that reads at most $\frac{1}{2} n$ entries of $A$ in the worst case and no more than $m$ in expectation, cannot determine the index of $x$ with probability greater than $\frac{2m}{n}$.
\end{lemma}

\begin{proof}
It is enough to prove the lemma for deterministic algorithms, as any randomized algorithm can be seen as a distribution over deterministic algorithms.
Fix some deterministic algorithm and denote the number of entries of $A$ it reads by the random variable $M$. We assume that $M \le n' = \floor{\frac{n}{2}}$ with probability one, and $\E[M] \le m$.
For each $t=1,\ldots,n'$, let $I_{t}$ be an indicator for the event that the $t$'th query of the algorithm is successful (we may assume that the entire sequence of queries of an algorithm is defined even when it terminates before actually completing it).
We can assume without loss of generality that the algorithm does not access an entry more than once (so that only one of its queries can be successful). Then, the algorithm's probability of success is given by $\E[\sum_{t=1}^{M} I_{t}]$.

Now, denote by $\set{\F_{t}}_{t=0}^{n'}$ the filtration generated by the algorithm's observations up to and including time $n'$ (with $\F_0 = \emptyset$), and observe that for all $1 \le t \le n'$ we have $\E[I_{t} \mid \F_{t-1}] \le \frac{1}{n-t+1} \le \frac{2}{n}$: if the algorithm was successful in its first $t-1$ queries then certainly $I_{t} = 0$; otherwise, the conditional expectation equals $\frac{1}{n-t+1}$ as the non-zero value $x$ has the same probability of being in any of the $n-t+1$ remaining entries (given any previous observations made by the algorithm). 

Define a sequence of random variables according to $Z_{t} = \sum_{s=1}^{t} (I_{t} - \E[I_{t} \mid \F_{t-1}])$ for $t=1,\ldots,n'$, and notice that $Z_{1},\ldots,Z_{n'}$ is a martingale with respect to $\set{\F_{t}}$, as $Z_{t}$ is measurable with respect to $\F_{t}$ and a simple computation shows that $\E[Z_{t} \mid \F_{t-1}] = Z_{t-1}$.
Also, observe that by definition $M$ is a stopping time with respect to $\set{\F_{t}}$, since the algorithm can only choose to stop based on its past observations. 
Hence, Doob's optional stopping time theorem \citep[see, e.g.,][]{levin2009markov} shows that $\E[Z_{M}] = \E[Z_{0}] = 0$. 
This implies that
\begin{align*}
	\Pr(\text{success})
	\eq \E\left[ \sum_{t=1}^{M} I_{t} \right] 
	\eq \E\left[ \sum_{t=1}^{M} \E[I_{t} \mid \F_{t-1}] \right] 
	\leq \E\left[ M \cdot \frac{2}{n} \right]
	\leq \frac{2m}{n}
	~,
\end{align*}
which completes the proof.
\end{proof}

We can now prove \cref{thm:online-lb}.

\begin{proof}[Proof of \cref{thm:online-lb}]
In what follows, we say that an algorithm \emph{touches} row index~$i$ if the algorithm calls, at some point throughout its execution, the oracle $\Oval$ with row index $i$ as input. 
We say that the algorithm touches column index $i$ if it invokes the oracle $\Oopt$ with a distribution which is supported on $i$.
Finally, we say that an algorithm touches index $i$ if it touches either the row index $i$ or the column index $i$. 
To lower bound the runtime of a given online algorithm, it is therefore enough to lower bound the total number of \emph{distinct} indices it touches.

We first observe that any algorithm that touches $m$ distinct indices can be implemented without invoking the oracle $\Oopt$ at all, such that the total number of \emph{row indices} it touches is at most $m$.
In other words, we can implement $\Oopt$ using the oracle $\Oval$ such that the total number of row indices touched by the resulting algorithm is no more than $m$.
To see this, recall \cref{lem:maxsupp1} that asserts that for any distribution $p$ over columns, one of the indices in $\supp(p)$ must be a valid answer to the query~$\Oopt(p)$.
Thus, to compute~$\Oopt(p)$ it is enough to simply read the entire rows whose indices are in $\supp(p)$ using repeated queries to $\Oval$, and manually compute the best performing expert over $p$.
Notice that the total number of distinct row indices touched by this implementation is indeed no more than the total number of different indices in the supports of all input distributions to $\Oopt$, which is at most $m$.

Hence, up to multiplicative constants in our bounds, we may restrict our attention to algorithms that do not use the optimization oracle $\Oopt$ at all, and lower bound the number of distinct row indices they touch.
Consider such an algorithm that attains average regret $<1$ on some round $t \le n$ with probability at least $\half$ on the randomized construction we described. Notice that this property is essential for the expected average regret to be $< \half$ due to Markov's inequality, so it is enough to focus exclusively on such algorithms.
We will show that the expected number of distinct row indices the algorithm touches, and hence its expected runtime, is at least $\Omega(\sqrt{N})$.

Denote by~$m$ the expected total number of distinct row indices touched, and for all $i=1,\ldots,n$, let $m_{i}$ be the expected number of distinct row indices from the set $X_{i}$ the algorithm touches.
Then, we have $m \ge \sum_{i=1}^{n} m_{i}$ as $X_{1},\ldots,X_{n}$ is a partition of $[N]$.
For all $i=1,\ldots,n$, let $p_{i}$ be the probability that the algorithm picked expert $x_{i}^{*}$ on one of the rounds $1,\ldots,i$ of the game. 
Since detecting one of the good experts on time is necessary for obtaining a sublinear regret, the algorithm's probability of attaining an average regret~$<1$ is upper bounded by $\sum_{i=1}^{n} p_{i}$.
This means that $\sum_{i=1}^{n} p_{i} \ge \frac{1}{2}$, as we assume that the algorithm succeeds with probability at least $\half$.

On the other hand, observe that \cref{lem:folklore} implies $p_{i} \le \frac{2m_{i}}{n}$ for each $i$, as any algorithm that makes no more than $m_{i}$ queries in expectation (and no more than $\frac{1}{2} n$ in the worst case) to experts in the range $X_{i}$ in the table of losses, cannot detect expert $x_{i}^{*}$ (that was chosen uniformly at random from this range) before round $i$ with probability higher than $\frac{2m_{i}}{n}$; notice that queries to other ranges in the table are irrelevant to this probability, since these ranges are constructed independently of $X_{i}$.
Hence, we obtain $\half \le \sum_{i=1}^{n} p_{i} \le \sum_{i=1}^{n} \frac{2m_{i}}{n}  \le \frac{2m}{n}$, from which we conclude that $m \ge \frac{1}{4} n = \frac{1}{4} \sqrt{N}$.
This concludes the proof.
\end{proof}

\section{Lower Bound for Online Binary Classification}

In this section we extend our results to the setting of online binary classification. In this setting, the actions of the adversary are pairs $(x,y)$ of a feature vector $x$ and a binary label $y \in \set{0,1}$. The loss function $\ell$ then has additional structure: the loss of any expert (or hypothesis, in the context of classification) over the example $(x,y)$ must be opposite to the loss of the same expert over the example $(x,1-y)$. The optimization oracle $\Oopt$ in this case receives a distribution over examples and emits the corresponding empirical risk minimizer---the hypothesis having the minimal loss with respect to the input distribution.

Since online binary classification is a special case of the optimizable experts setting (we merely impose additional constraints on the loss function), our algorithms and runtime upper bounds directly transfer to this specific case.
However, the lower bounds do not directly apply: our constructions of loss functions for the proofs of the lower bounds (in both \cref{sec:lower} and \cref{sec:online-lower}) do not necessarily admit the additional structure required by a loss function in the binary classification setting.
Nevertheless, below we show that the construction given in \cref{sec:online-lower} can be adapted to binary classification, and thereby reprove the $\Omega(\sqrt{N})$ runtime lower bound in the latter setting.

First, let us define the setting more formally.
In online binary classification, there is a finite set $\H$ of $N$ hypotheses, a set $\X$ of feature vectors, and a loss function $\ell : \H \times (\X \times \set{0,1})$ that assigns losses to all pairs of hypothesis $h \in \H$ and labeled example $(x,y) \in \X \times \set{0,1}$.
First, an adversary privately chooses an arbitrary sequence $(x_1,y_1),\ldots,(x_T,y_T) \in \X \times \set{0,1}$ of labeled examples.
Then, on each round $t=1,\ldots,T$, the player receives the feature vector $x_t$ and has to pick an hypothesis $h_t \in \H$, possibly at random; subsequently, the player suffers the loss $\ell(h_t;x_t,y_t)$ and observes the label $y_t$.%
\footnote{The hypothesis of choice $h_t$ is typically used to classify $x_t$ via a classification rule $\phi$, and the incurred loss is then a function of the classification $\phi(h_t,x_t)$ and the true label $y_t$. This is equivalent to our definition: any binary loss function $\ell(h;x,y)$ can be equivalently written as $\t\ell(\phi(h,x),y)$ with a suitable $\phi : \H \times \X \mapsto \set{0,1}$.}
The goal of the player is to minimize the running time required to achieve $\eps$ expected average regret, namely to reach
\begin{align*}
	\EE{ \frac{1}{T} \sum_{t=1}^T \ell(h_t;x_t,y_t) } 
	~-~ \min_{h \in \H} \frac{1}{T} \sum_{t=1}^T \ell(h;x_t,y_t)
	\leq \eps ~.
\end{align*}
The oracles $\Oval$ and $\Oopt$ are defined exactly as before: the value oracle satisfies $\Oval(h;x,y) = \ell(h;x,y)$ for all $h \in \H$ and $(x,y) \in \X \times \set{0,1}$; the optimization oracle accepts a distribution $p \in \Del(\X \times \set{0,1})$ and returns the hypothesis $h \in \H$ that minimizes $\sum_{(x,y)} p(x,y) \ell(h;x,y)$.

In the (optimizable) online binary classification model, we prove:

\begin{theorem}
\label{thm:binary-lower}
Let $N>0$ and fix $\H = \X = [N]$.
For any (randomized) regret minimization algorithm, there is a loss function~$\ell : \H \times (\X \times \set{0,1}) \mapsto [0,1]$, corresponding oracles $\Oval$, $\Oopt$, and a sequence of labeled examples $(x_1,y_1),(x_2,y_2),\ldots \in \X \times \set{0,1}$ such that the runtime required for the algorithm to attain expected average regret smaller than $\half$ is at least $\Omega(\sqrt{N})$.
\end{theorem}

In order to prove the theorem, we adapt the construction of \cref{sec:online-lower} as follows.
Assume that $N = n^{2}$ for some integer $n>1$, and let $\H = [N]$ be the hypothesis class and $\X = [N]$ be the set of possible feature vectors.
Pick a set $H^{*} = \set{h_{1}^{*},\ldots,h_{n}^{*}} \subseteq \H$ of $n$ ``good'' hypotheses, by choosing $h_{i}^{*} \in H_{i} = \set{n(i-1)+1,\ldots,ni}$ uniformly at random for each $i=1,\ldots,n$.
Also, for each feature vector $x \in X$ choose a ``good'' label $y^*(x) \in \set{0,1}$ uniformly at random.
Then, define losses for all pairs of hypothesis $h \in \H$ and example $(x,y) \in \X \times \set{0,1}$ via:
\begin{align*}
	\ell(h;x,y)
	\eq \mycases
		{\t\ell(h,x)}{if $y = y^*(x)$,}
		{1-\t\ell(h,x)}{if $y \ne y^*(x)$,}
\end{align*}
where $\t\ell$ is the loss function constructed in \cref{sec:online-lower}, namely:
\begin{align*}
	\t\ell(h,x) 
	\eq \mycases{0}{if $h,x \in H^{*}$ and $h \ge x$,}{1}{otherwise.}
\end{align*}
%
%
For this construction we can prove the next theorem, from which \cref{thm:binary-lower} immediately follows.

\begin{theorem} \label{thm:binary-lb}
Consider an adversary that picks the sequence $(x_1,y_1),\ldots,(x_n,y_n)$ of examples, where $x_{t} = h_t^*$ and $y_t = y^*(x_t)$ for all~$t=1,\ldots,n$.
Any online algorithm whose expected runtime is $o(\sqrt{N})$ cannot attain expected average regret smaller than $\tfrac{1}{4}$ (at some point $1 \le t \le n$ during the game) on the sequence~$(x_1,y_1),\ldots,(x_n,y_n)$.
\end{theorem}

The proof of \cref{thm:binary-lb} is very similar to the proof of \cref{thm:online-lb}: the only difference is in \cref{lem:maxsupp1} that no longer applies for the new construction. 
However, we can prove the following analogue of that lemma in our current setup.

\begin{lemma} \label{lem:maxsupp2}
For any distribution $p \in \Del(\X \times \set{0,1})$ over examples, one of the following statements must hold true:
(i) there exists $(x^*,y^*) \in \supp(p)$ such that $h = x^*$ is a valid answer to $\Oopt(p)$;
(ii) any $h \in \H \setminus H^*$ is a valid answer to the query $\Oopt(p)$.
\end{lemma}

Intuitively, the lemma tells us that for the loss function $\ell$ we constructed, the optimization oracle is completely redundant, as the output of a query $\Oopt(p)$ can be implemented via a manual search over the support of $p$. In other words, the optimization oracle does not reduce the number of hypotheses we would have to inspect for minimizing the regret.

\begin{proof}
Let $S = \set{x : (x,y) \in \supp(p)}$.
Notice that if $S \cap H^* = \emptyset$, i.e., $p$ does not hit any of the good feature vectors (that correspond to the good hypotheses), then any hypothesis is a valid answer to $\Oopt(p)$, since for any $x \in S$ and $y \in \set{0,1}$ it holds that $\ell(h;x,y) = \ell(h';x,y)$ for all $h,h' \in \H$. In particular, any $h \in S$ is valid answer to $\Oopt(p)$ in this case.
In all other cases, it is enough to consider only the elements in the intersection $S' = S \cap H^*$, 
since atoms $(x,y) \in \supp(p)$ with $x \notin H^*$ contribute to the losses of all hypotheses equally and do not affect the optimization for the best hypothesis with respect to $p$. 

For all $1 \le i \le n$ let $p_i = p(x_i^*,y^*(x_i^*))$ and $p_i' = p(x_i^*,1-y^*(x_i^*))$. 
Also, for notational convenience, let $x_0^*$ denote an arbitrary hypothesis from $\H \setminus H^*$.
Then, inspecting the structure of~$\ell$, it follows that $x_j^*$ is a valid answer to $\Oopt(p)$, where
\begin{align*}
	i^* \eq 
	\argmin_{0 \le i \le n} 
	\lrset{ p_1' + \cdots + p_i' + p_{i+1} + \cdots + p_n }
	~,
\end{align*}
and in case there are multiple minimizers, $\Oopt$ chooses the one with smallest $i^*$.
Consider two cases: $i^*=0$ and $i^* \ge 1$. In the first case, $x_0^*$ is a valid answer to $\Oopt(p)$, and the lemma's claim holds true since $x_0^*$ can be any hypothesis from $\H \setminus H^*$.
In the second case, we claim that it must be the case that $p_{i^*} > 0$: otherwise, we would have
\begin{align*}
	p_1' + \cdots + p_{i^*-1}' + p_{i^*}' + p_{i^*+1} + \cdots + p_n
	\geq
	p_1' + \cdots + p_{i^*-1}' + p_{i^*} + p_{i^*+1} + \cdots + p_n
	~,
\end{align*}
which contradicts the optimality and minimality of $i^*$.
This means that $x_{i^*}^* \in S' \subseteq S$, and $x_{i^*}^*$ is a valid response to $\Oopt(p)$, which gives the lemma.
\end{proof}

\cref{thm:binary-lb} now follows via the same arguments we used for proving \cref{thm:online-lb}, using \cref{lem:maxsupp2} in place of \cref{lem:maxsupp1}.

\end{document}